\documentclass[12pt]{article}
\usepackage[textwidth=6in,top=1in,bottom=1in]{geometry}
\usepackage{amsmath,amsthm,amssymb,amsopn,amsfonts,pdfpages,algorithmic,algorithm,dsfont}
\usepackage{graphics}
\usepackage{graphicx}
\usepackage{bbm}
\usepackage{caption,natbib,url}
\usepackage[colorlinks=true,linkcolor=cyan,citecolor=blue]{hyperref}
\usepackage{stackengine}
\usepackage{authblk}
\usepackage[shortlabels]{enumitem}
\usepackage[pagewise]{lineno}
\usepackage{subcaption}
\newcommand{\e}{\mathbb{E}}
\newcommand{\p}{\mathbb{P}}

\newcommand{\hpi}{\widehat{{P}}_n}
\newtheorem{theorem}{Theorem}

\newtheorem{example}[theorem]{Example}
\newtheorem{corollary}[theorem]{Corollary}
\theoremstyle{definition}
\newtheorem{definition}[theorem]{Definition}
\theoremstyle{remark}
\newtheorem{remark}[theorem]{Remark}

\DeclareMathOperator*{\argmin}{arg\,min}
\DeclareMathOperator*{\argmax}{arg\,max}

\begin{document}

\def\spacingset#1{\renewcommand{\baselinestretch}%
{#1}\small\normalsize} \spacingset{1}

%\linenumbers

%%%%%%%%%%%%%%%%%%%%%%%%%%%%%%%%%%%%%%%%%%%%%%%%%%%%%%%%%%%%%%%%%%%%%%%%%%%%%%

	\title{\bf Maximum Likelihood Estimation and Graph Matching in Errorfully Observed Networks}
  \author[$1$]{Jes\'us Arroyo}
  \author[$2$]{Daniel L. Sussman}
  \author[$1,3$]{Carey E. Priebe}
  \author[$4$]{Vince Lyzinski}
  \affil[$1$]{\small Center for Imaging Science, Johns Hopkins University}
  \affil[$2$]{\small Department of Mathematics and Statistics, Boston University}
  \affil[$3$]{\small Department of Applied Mathematics and Statistics, Johns Hopkins University}
  \affil[$4$]{\small Department of Mathematics, University of Maryland}
  \date{}
  \maketitle

\begin{abstract}
Given a pair of graphs with the same number of vertices, the inexact graph matching problem consists in finding a correspondence between the vertices  of these graphs that minimizes the  total number of induced edge disagreements. 
We study this problem from a statistical framework in which one of the graphs is an errorfully observed copy of the other. 
We introduce a corrupting channel model, and show that in this model framework, the solution to the graph matching problem is a maximum likelihood estimator (MLE). 
Necessary and sufficient conditions for consistency of this MLE are presented, as well as a relaxed notion of consistency in which a negligible fraction of the vertices need not be matched correctly. 
The results are used  to study matchability in several families of random graphs, including edge independent models, random regular graphs and small-world networks. 
We also use these results to introduce measures of matching feasibility, and experimentally validate the results on simulated and real-world networks.
\end{abstract}

\noindent%
{\it Keywords:}  Graph matchability, corrupting channel, consistency
\vfill

\newpage
\spacingset{1} % DON'T change the spacing!
%%%%%%%%%%%%%%%%%%%%%%%%%%%%%%%%%%%%%%%%%%%%%%%%%%%%%%%5

\section{Introduction}
%Graphs are a popular data structure used to represent relationships between objects or agents, with successful applications in many different fields, including finance \citep{fin1,fin2,fin3}, neuroscience \citep{neu1,neu2,neu3}, biology \citep{bio1,bio2,bio3}, and sociology \citep{socnet1,socnet2,resp1,resp2}, among others.  Many applications deal with multiple graph observations that can come as different instances of the graph for the same or related set of vertices, and thus studying these data jointly usually requires knowledge of the correspondence between the vertices first; see, for example, \citep{dunson1,eynard2015multimodal,lyzinski2017fast,tang2017semiparametric}. When this correspondence is unknown or observed with errors, graph matching can be used to first recover the true correspondence before performing subsequent inference. Simply stated, the \emph{graph matching problem} consists in finding an alignment between the vertices of two different graphs that minimizes a measure of dissimilarity between them (usually specified to be the number of edge differences).  This problem has found many useful applications in different areas including de-anonymizing social networks \citep{narayanan2009anonymizing,korula2014efficient, final,regal}, aligning biological networks \citep{gmbio1,gmbio2,VogConLyzPodKraHarFisVogPri2015}, and unsupervised word translation \citep{grave2018unsupervised}, among others; see the surveys \cite{ConteReview,foggia2014graph,yan2016short} for a review (up to 2016) of the graph matching literature and applications.

Graphs are a popular data structure used to represent relationships between objects or agents, with successful applications in many different fields, including finance, neuroscience, biology, and sociology, among others. Many applications deal with multiple graph observations that can come as different instances of the graph for the same or related set of vertices, and thus studying these data jointly usually requires knowledge of the correspondence between the vertices first. When this correspondence is unknown or observed with errors, graph matching can be used to first recover the true correspondence before performing subsequent inference. Simply stated, the \emph{graph matching problem} consists in finding an alignment between the vertices of two different graphs that minimizes a measure of dissimilarity between them (usually specified to be the number of edge differences).  This problem has found many useful applications in different areas including de-anonymizing social networks \citep{narayanan2009anonymizing,korula2014efficient, final,regal}, aligning biological networks \citep{gmbio1,gmbio2,VogConLyzPodKraHarFisVogPri2015}, and unsupervised word translation \citep{grave2018unsupervised}, among others; see the surveys \cite{ConteReview,foggia2014graph,yan2016short} for a review (up to 2016) of the graph matching literature and applications.

\subsection{The graph matching problem}
\label{sec:GMP}
Formally, given a pair of adjacency matrices for some graphs $A = (E_1, V_1)$ and $B=(E_2,V_2)$ with equal number of vertices $|V_1|=|V_2|$ (where $|V|$ denotes cardinality of a set of vertices $V$), the graph matching problem (GMP) consists of finding a mapping $\pi:V_1\rightarrow V_2$ which aligns the vertices between the two graphs to make the structure most similar; i.e., if $\Pi_n$ is the set of $n\times n$ permutation matrices, then the GMP  is $\argmin_{P\in\Pi_n}\|A-PBP^T\|_F$, where $\|E\|_F=(\sum_{i=1}^n\sum_{j=1}^nE^2_{ij})^{1/2}$ is the Frobenius norm of some matrix $E\in\mathbb{R}^{n\times n}$.
Variants of this basic formulation have been proposed in the literature to accommodate more nuanced graph structure including, incorporating weighted and directed edges, $|V_1|\neq|V_2|$ \citep{fishkind2018alignment}, 
vertex and edge covariates \citep{lyzinski2016consistency}, and matching multiple graphs simultaneously \citep{yan2013joint}.

A special case of the problem is
exact graph matching, also known as the graph isomorphism problem \citep{babai2015graph}, where the goal is to determine if there exists an alignment between the vertices of the graphs yielding identical edge structure across networks. 
Even for this special case, 
it is currently not known whether the problem is solvable in polynomial time.
However, an exact unique solution exists as long as there is no non-identity automorphisms (i.e., $\forall P\neq I_n$, $A\neq PAP^T$) . 
In practice, the applicability of exact graph matching is limited due to the fact that data often consists of noisy observations of a graph, and thus one can only hope to recover an alignment that preserves a significant portion of the structure.

In the inexact version of the graph matching problem, a pair of graphs with the same number of vertices is observed, and the goal is to find an alignment of the vertices that best preserves the structure of the graphs.
This is often accomplished by minimizing an appropriate dissimilarity measure over all the possible permutations (e.g., the Frobenius norm formulation considered above). 
% Several random graph models have been used to study this problem. These models usually propose a  mechanism that results in a pair of graphs in which there is a unobserved alignment between the vertices for which the graphs $A$ and $B$ look similar. The 
%To better understand the difficulty and feasibility of this problem, 
Several random graph models have been proposed in the recent literature, and these have been used to study the difficulty and feasibility of this problem. 
These models often parametrically enforce a natural similarity between the graphs while still allowing for structural differences.
They vary in complexity from correlated homogeneous Erd\H os-R\'enyi networks \citep{pedarsani2011privacy,yartseva2013performance,JMLR:v15:lyzinski14a}, to correlated stochastic blockmodel networks \citep{onaran2016optimal,lyzinski2018}, to correlated general edge independent networks \citep{korula2014efficient,rel}, to independent graphon generated networks \citep{zhang2018unseeded}.
Within these models, the dual problems of developing efficient graph matching algorithms and studying the theoretical feasibility of the GMP have been considered.
However, beyond (conditionally) edge-independent networks, few theoretical guarantees exist for either algorithmic performance (see \cite{korula2014efficient} for an example of matching guarantees in a preferential attachment model).

\subsection{Graph matchability and maximum likelihood estimation}
\label{sec:MLEGM}

These random graph models have allowed for the exploration of the related notion of graph matchability:  given a natural correspondence between the vertex sets of two networks, can the GMP uncover this correspondence in some statistical sense \citep{JMLR:v15:lyzinski14a,rel,lyzinski2016information,onaran2016optimal,cullina2017exact}?
In this paper, we cast the problem of graph matchability in the framework of maximum likelihood (ML) estimation, equating matchability with the consistency of the maximum likelihood estimator (MLE) of an unknown correspondence.
The results bear a similar flavor to those in \cite{onaran2016optimal} (and \cite{cullina2016improved}), wherein a model for correlated stochastic blockmodels is proposed.
In their framework, they showed that maximum a posteriori (MAP) estimation is the same as optimizing a weighted variant of the classical graph matching problem.
Unlike previous work, our model (see Section \ref{sec:ucc}) is designed to be distribution-free, allowing for corrupted or correlated observations of arbitrarily structured networks to be considered.
We note that while we are not the first to frame the GMP as a ML estimation problem (see, for example, \cite{luo2001structural,lyzinski2016consistency}), 
our model allows for a novel understanding of the relationship between the two, seemingly disparate, ideologies.

While maximum likelihood estimation is a core concept in modern statistical inference \citep{stigler2007epic}, estimation of an unknown correspondence between two networks presents the challenge of ML estimation in the presence of a growing parameter dimension. Our model is parameterized by the shuffling permutation, which is the parameter of interest, and corrupting probabilities for each edge that are nuisance parameters.
Indeed, viewing the correspondence between vertices as a parameter to be estimated, this parameter naturally contains $n-1$ free variables, %is situated in $\mathbb{R}^{n-1}$
where $n$ is the number of vertices in the observed network, while the number of nuisance parameters can grow as $O(n^2)$.
Asymptotics (as $n\rightarrow\infty$, yielding $\binom{n}{2}$ sample size) resist the classical theoretical framework when the number of (nuisance) parameters grow with the sample size (see, for example, \cite{bickel2015mathematical}), and the MLE can be inconsistent in this  setting \citep{neyman1948consistent, lancaster2000incidental}.
Recently, statistical network inference (see, for example, \cite{BicChe2009,zhao2012consistency}) has presented further examples of consistent ML estimation in networks when the parameter dimension is growing, notably for the inference task of community detection in stochastic blockmodel random graphs.
% Moreover, unlike in the settings of 
By equating consistent ML estimation with graph matchability, we provide another class of examples in the network literature for which consistency of the MLE is achieved as the graph size (and parameter dimension) increases.

%\textcolor{red}{add a bit more here on classical mle in growing parameter setting?}         

%Consistency of graph matching. Different definitions and results.  ``Risk" consistency \citep{zhang2018unseeded,zhang2018consistent}: minimize number of disagreements but no guarantee on the permutation matrices.

The paper is organized as follows. In Section \ref{sec:ucc} we introduce the  corrupting channel model and show that maximum likelihood estimation of the latent vertex correspondence is a solution to the graph matching problem (and vice versa). Next, in Section \ref{sec:consis} we derive necessary and sufficient conditions (depend on the structure of the given graph and the channel noise) for consistency of the MLE.
In Section \ref{sec:quasi} we introduce a notion of quasi-consistency in which we allow a fraction of the vertices to be incorrectly matched, and present sufficient conditions to achieve this property.
Results of consistency and quasi-consistency of the MLE in a variety of random graph models are
presented, including some new results for small-world and random regular graphs. In Section \ref{sec:experiments} we introduce a practical approach to studying matchability, and validate our theoretical results with numerical experiments  on simulated networks from popular random graph models and real-world data from different domains. We conclude with a discussion and remarks in Section \ref{sec:discussion}.

% JA: changed 
\section{Equivalence between maximum likelihood estimation and graph matching}

\subsection{Uniformly corrupting channel}
\label{sec:ucc}

Consider the set of undirected graphs with no self-loops and $n$ labeled vertices, which we denote by $\mathcal{G}_n$. Each graph will be represented with an adjacency matrix $A\in\{0,1\}^{n\times n}$, which is symmetric and hollow (i.e., with zeros in the diagonal). Let $p\in[0,1]$, $P\in\Pi_n=\{n\times n \text{ permutation matrices}\}$, and $A\in\mathcal{G}_n$. We model passing $A$ through an edge and vertex-label corrupting channel as follows.
\begin{definition}
	\label{def:corrupted}
	Define the graph-valued random variable
	$B_{p,P}$ (i.e., the channel-corrupted $A$)
	parameterized by $(p,P)$ via
	$B_{p,P}:=P(X\circ (1-A)+(1-X)\circ A)P^T$,
	where $X\in\mathbb{R}^{n\times n}$ is a symmetric, hollow matrix with i.i.d. $\text{Bernoulli}(p)$ entries in its upper triangular part, $X$ is independent of $A$, and $``\circ"$ is the Hadamard matrix product.
	For ease of notation, we shall write $B_{p,P}=B\sim$C$(A,p,P)$ for $B$ distributed as the channel corrupted $A$.
\end{definition}
\noindent 
Note that if $A$ is a graph-valued random variable, then we further assume that $X$ is independent of $A$.
This model is similar to the noisily observed network models of \cite{pedarsani2011privacy,yartseva2013performance,korula2014efficient,kola2018estimation}.  
Note that here we make no assumptions on the distribution of the underlying graph $A$; indeed, below we often view $A$ as deterministic and not random, and consider the distribution of $B$ conditional on $A$. 
% When $A$ is distributed according to the Erd\H{o}s-R\'enyi graph model. 

Stated simply, $B_{p,P}$ is formed by first uniformly corrupting (i.e., bit-flipping) edges in $A$ independently with probability $p$ and then shuffling the labels via $P$.
Given observed $(A,B)$, the likelihood of $(p,P)$ in this model is given by
$$L(p,P)=\prod_{i<j} (1-p)^{\mathds{1}\{A_{i,j}=(P^TBP)_{i,j} \}}p^{\mathds{1}\{A_{i,j}\neq(P^TBP)_{i,j} \}},$$
so that the log-likelihood is given by
\begin{align*}
\ell(p,P)&=\sum_{i<j}\mathds{1}\{A_{i,j}=(P^TBP)_{i,j} \}\log(1-p)+\mathds{1}\{A_{i,j}\neq(P^TBP)_{i,j} \}\log p\\
&=\frac{1}{2}\text{tr}(AP^TBP)\log(1-p) + \frac{1}{2}\left(n(n-1) -\text{tr}(AP^TBP)\log p\right)\\
&= \frac{1}{2} \text{tr}(AP^TBP)(\log(1-p)-\log p) + \frac{n(n-1)\log p}{2},
\end{align*}
and given a fixed value of $p$, observing that $2\text{tr}(AP^TBP)=\|A\|_F^2+\|B\|_F^2 - \|A-P^TBP\|_F^2$, the (possibly non-unique) MLE of $P$ is then given by
%$$\hat P=\begin{cases}
%\text{any }P\in\argmax_{P\in\Pi_n}\text{tr}(AP^TBP)&\text{ for }p<1/2\\
%\text{any }P\in\argmin_{P\in\Pi_n}\text{tr}(AP^TBP)&\text{ for }p>1/2\\
%\text{any }P\in\Pi_n&\text{ for }p=1/2
%\end{cases} $$
%Equivalently,
$$\hat P_{\text{MLE}}=\begin{cases}
\text{any }P\in\argmin_{P\in\Pi_n}\|A-P^TBP\|_F^2&\text{ for }p<1/2\\
\text{any }P\in\text{argmin}_{P\in\Pi_n}\|A-P^T\breve{B}P\|_F^2 & \text{ for }p>1/2\\
\text{any }P\in\Pi_n&\text{ for }p=1/2,
\end{cases} $$
where $\breve{B}$ is the complement of the graph $B$ (i.e., $\breve{B}_{ij}=1-B_{ij}$). In the $p<1/2$ setting, the problem of maximum likelihood estimation is equivalent to the problem of graph matching, as the classical GMP formulation is to find the (possibly non-unique) minimizer of a quadratic assignment problem (QAP) of the form
\begin{equation}
 \hat{P}_{\text{GM}} \in \argmin_{P\in\Pi_n}\|A-P^TBP\|_F^2  \label{eq:GMP}
\end{equation}
Analogously, when $p>1/2$, ML estimation is equivalent to matching one graph with the complement of the other. 

A natural first question to ask is what properties 
$A$ (and $p$) must possess in order for the MLE for $P$ to be consistent, and hence the vertex-label corruption introduced into $B$ can be undone via graph matching. 
We note here that the parameter dimension of $P$ is growing in $n$ and classical ML estimation consistency results do not directly apply (see, for example, \cite{bickel2015mathematical}).
Nonetheless, in Section \ref{sec:consis} our main result, Theorem \ref{thm:asym}, will establish consistency of the MLE under fairly modest assumptions on $A$ and $p$. 

\vspace{3mm}

% \begin{remark}
% \label{rem:pandq}

% 	Consider the setting where edges in $A$ are corrupted by the channel independently with probability $p$ and non-edges in $A$ are corrupted by the channel independently with probability $q$.
% In this setting, the channel corrupted $A$ can be defined as follows: $B_{p,q,P}=B$ (i.e., the channel-corrupted $A$)
% parameterized by $(p,q,P)$ is defined via
% $$B=P( (1-X)\circ A+Y\circ(1- A) )P^T,$$
% where $X\in\mathbb{R}^{n\times n}$ is a symmetric, hollow matrix with i.i.d. ${Bern}(p)$ entries in its upper triangular part; $Y\in\mathbb{R}^{n\times n}$ is a symmetric, hollow matrix with i.i.d. ${Bern}(q)$ entries in its upper triangular part; $X$ and $Y$ are mutually independent; and $``\circ"$ is the Hadamard matrix product.
% As derived above, it follows that the MLE for $P$ is given by
% $$\hat P=\begin{cases}
% \text{any }P\in\text{argmax}_{P\in\Pi_n}\text{tr}(AP^TBP)&\text{ for }p+q<1\\
% \text{any }P\in\text{argmin}_{P\in\Pi_n}\text{tr}(AP^TBP)&\text{ for }p+q>1\\
% \text{any }P\in\Pi_n&\text{ for }p+q=1
% \end{cases} $$

% \end{remark}

\subsection{Heterogeneous corrupting channel\label{sec:hetchannel}}

The uniform corrupting channel model defined above assumes that all edges of $A$ are corrupted with the same probability. However, in some cases it is more reasonable to consider a model in which each pair of vertices might be corrupted with different probabilities. For example, we can consider the setting in which edges and non-edges in $A$ are corrupted by the channel independently with different probabilities, or in which certain vertices or edges are more likely to be corrupted. 
Thus, we also consider an heterogeneous corrupting channel defined as follows.

\begin{definition}
	Let $\Psi^{(m)}\in[0,1]^{n\times n}, m=1,2$  be two matrices that correspond to the corrupting probabilities for edges and non-edges, and $P\in\Pi_n$. For a given adjacency matrix $A\in\{0,1\}^{n\times n}$, we define $B_{\Psi^{(1)},\Psi^{(2)},P}$ as
	\[B_{\Psi^{(1)},\Psi^{(2)},P}=P( X \circ (1-A)+ (1-Y) \circ A)P^T,\]
	where $X$ and $Y$ are symmetric hollow matrices such that $X_{ij}\sim \text{Bernoulli}(\Psi^{(1)}_{ij})$ and $Y_{ij} \sim \text{Bernoulli}(\Psi^{(2)}_{ij})$ and $\{X_{i,j},Y_{i,j}\}_{i<j}$ are independent. 
	Dropping subscripts to ease notation when possible, we denote it via $B \sim \text{C}(A, \Psi^{(1)}, \Psi^{(2)}, P)$ for $B$ distributed as the heterogeneous corrupting channel $A$.\label{def:corrupted-het}
\end{definition}
\noindent Note that if $A$ is a graph-valued random variable, then we further assume that $X$ and $Y$ are independent of $A$. The heterogeneous corrupting channel model is saturated, but the only parameter of interest is $P$.

The flexibility of having different corrupting probabilities for each edge allows us to include other popular graph models within this framework. 
In particular, if $A\sim\text{Bernoulli}(\Lambda)$, the heterogeneous corrupting channel model can be used to describe the distribution of a pair of correlated Bernoulli graphs conditional on the first graph. 
In more detail, given a pair of hollow symmetric matrices $R,\Lambda\in[0,1]^{n\times n}$, the pair of graphs $A,B$ is said to be distributed as $R$-correlated random Bernoulli$(\Lambda)$ graphs if marginally $A, B\sim \text{Bernoulli}(\Lambda)$ (i.e., for $1\leq i<j\leq n$, $A_{ij}$ are independently distributed as $\text{Bern}(\Lambda_{ij})$ random variables; similarly for $B$) and each pair of variables $A_{ij},B_{\ell k}$,  $1\leq i<j\leq n, 1\leq\ell<k\leq n$, $(i,j)\neq (\ell, k)$ are mutual independent, except that for each $1\leq i<j\leq n$, $\text{corr}(A_{ij},B_{ij}) = R_{ij}$.  Note that in the correlated Bernoulli graphs model, $\Bbb{P}(B_{ij}=1|A_{ij}) = A_{ij}R_{ij} + (1-R_{ij})\Lambda_{ij}$, so the distribution of $B$ conditioned on $A$ can be written using an heterogeneous channel with %$\Psi^{(1)}_{ij} = (1-R_{ij})(1-\Lambda_{ij})$ and $\Psi^{(2)}_{ij} =  (1-R_{ij})\Lambda_{ij}.$ 
\begin{eqnarray}
    \Psi^{(1)}_{ij} & = & (1-R_{ij})\Lambda_{ij} \label{eq:corr-psi1}\\
    \Psi^{(2)}_{ij} & = & (1-R_{ij})(1-\Lambda_{ij}). \label{eq:corr-psi2}
\end{eqnarray}
This correlated Bernoulli graph model, and in particular, the correlated Erd\H{o}s-R\'enyi (ER) model (in which  each of the matrices $\Lambda$ and $R$ have the same value in all the entries), have been extensively used in studying the inexact graph matching problem \citep{pedarsani2011privacy,lyzinski2014seeded,rel,cullina2017exact, ding2018efficient,mossel2019seeded}. The corrupting channel model is thus comprehensive enough to include other popular settings in literature (see for example Corollary \ref{cor:examples} for a result in correlated ER graphs), while allowing to extend some of the results to graphs with arbitrary structure.

In the heterogeneous corrupting channel, the MLE of the unshuffling permutation is also exactly equivalent to the minimizer of the graph matching objective \eqref{eq:GMP}. Note that in this case, the likelihood is given by
\begin{eqnarray*}
	\ell(\Psi^{(1)}, \Psi^{(2)},P) & = &  \sum_{i<j} \left\{ (P^T{B}P)_{ij}\log\left((1- \Psi^{(1)}_{ij})A_{ij} +\Psi^{(2)}_{ij}(1-A_{ij})\right) \right.\\
	&  & \left.   + (1-(P^T{B}P)_{ij})\log\left( \Psi^{(1)}_{ij}A_{ij} + (1-\Psi^{(2)}_{ij})(1-A_{ij}) \right) \right\}.
\end{eqnarray*}
In general, the MLE for $P$ in this model is not unique since the corrupting probabilities can arbitrarily alter the graph. Thus, we impose a further assumption that the probability of any pair of corresponding edges in the graphs $A$ and $B$ to be the same is at least 0.5. Formally, let $\tilde p = \max_{i,j,k}\Psi^{(k)}_{ij}$ be the largest corrupting probability for an edge. Under the assumption that $\tilde p \leq \frac{1}{2}$, for a given $P\in\Pi_n$ we can calculate a  profile MLE for $\Psi^{(k)}, k=1,2$, denoted by $\hat{\Psi}^{(k)}(P)$, as
\begin{equation}
\left(\hat{\Psi}^{(k)}(P)\right)_{ij} =
\begin{cases}
0 & \text{if $A_{ij}=(P^T{B}P)_{ij},$}  \\
\frac{1}{2} & \text{if $A_{ij}\neq  (P^T{B}P)_{ij}$} 
%\text{any } c\in[0,1/2] & \text{otherwise,}
\end{cases}. \label{eq:nuisance}
\end{equation}
Thus, the profile loglikelihood for $P$ can be expressed as
\begin{align*}
	\ell\left(\!\hat{\Psi}^{(1)}(P),\hat{\Psi}^{(2)}(P),P\!\right) & = \sum_{i<j} 2\mathds{1}(A_{ij}\neq(P^T{B}P)_{ij})  \log(1/2) \\
	 & =   -\log(2) \|A - (P^T{B}P)\|_F^2,
\end{align*}
which does not depend on the particular values of $\hat{\Psi}^{(1)}(P), \hat{\Psi}^{(2)}(P)$ and is proportional to the graph matching objective \eqref{eq:GMP}. Therefore, the (possibly non-unique) MLE for $P$, given by
\begin{equation*}
    \hat{P}_{\text{MLE}} \in \argmax_{P\in\Pi_n} \ell\left(\!\hat{\Psi}^{(1)}(P),\hat{\Psi}^{(2)}(P),P\!\right)
\end{equation*}
is exactly equivalent to the solution of the GMP.

\section{Consistency of the maximum likelihood estimator}
\label{sec:consis}

For a given positive integer $n_0$, consider a sequence of graphs $(A_n)_{n=n_0}^\infty$ (with the order of $A_n$ being $n$).
For the parameter sequence $(\,(p_n,P_n)\,)_{n=n_0}^\infty$, consider the sequence of channel-corrupted $A_n$'s defined via
$$\left(B_n=P_n(X_n\circ (1-A_n)+(1-X_n)\circ A_n)P_n^T \right)_{n=n_0}^\infty,$$
where each $X_n$ is a symmetric, hollow matrix with i.i.d. $\text{Bernoulli}(p_n)$ entries in its upper triangular part,
and the $X_n$'s are mutually independent across index $n$.
In this setting, we define consistency of a sequence of estimators as follows.
\begin{definition}[Consistency]
	\label{def:consis}
	With $(B_n\sim\text{C}(A_n,p_n,P_n))_{n=n_0}^\infty$ defined as above, for $n\geq n_0$ %let $$\widehat{\mathcal{P}}_n=\text{argmin}_{P\in\Pi_{n}}\|A_n-P^TB_nP\|_F^2$$ 
	%be the MLE for $P_n$ (equivalently, the solution of the GMP for $A_n$ and $B_n$). W
	we say that a sequence of estimators $(\hat P_n:=\hat P_n(A_n,B_n))_{n=n_0}^\infty$ of $(P_n)_{n=n_0}^\infty$, %where $\hat P_n\in\hpi$ for all $n\geq n_0$,
	is a $\begin{pmatrix}\text{weakly}\\\text{strongly} \end{pmatrix}$ consistent estimator if $\|\hat P_n-P_n\|_F\begin{pmatrix}\stackrel{P }{\rightarrow}\\ \stackrel{ a.s.}{\rightarrow} \end{pmatrix}0$ as $n\rightarrow\infty$, for all sequences $(P_n)_{n=n_0}^\infty$.
\end{definition}

\noindent As the MLE of $(P_n)_{n=n_0}^\infty$ is not necessarily uniquely defined (i.e., there are multiple possible MLEs for a given $n$), we write that the \emph{MLE of $(P_n)_{n=n_0}^\infty$ is 
$\begin{pmatrix}\text{weakly}\\\text{strongly} \end{pmatrix}$ consistent} if $\max_{Q_n\in \mathcal{P}^*_n}\|Q_n-P_n\|_F\begin{pmatrix}\stackrel{P }{\rightarrow}\\ \stackrel{ a.s.}{\rightarrow} \end{pmatrix}0$ as $n\rightarrow\infty$,
where
$$\mathcal{P}^*_n=\argmin_{P\in\Pi_n}\|A_n - P^TB_nP\|^2_F.$$
Note that, as $\max_{Q_n\in \mathcal{P}^*_n}\|Q_n-P_n\|^2_F$ is integer-valued, strong consistency of the MLE is equivalent to
\begin{align*}
& \p(\max_{Q_n\in \mathcal{P}^*_n}\|Q_n-P_n\|^2_F=0\text{ for all but finitely many }n)=1\\ 
&\Longleftrightarrow 
\p(\mathcal{P}^*_n=\{P_n\}\text{ for all but finitely many }n)=1.
\end{align*}

Loosely speaking,  a estimator
is consistent if the number of vertices that are incorrectly matched by the estimator goes to zero as the size of the graph increases.
It is clear that consistency of an estimator hinges on the properties of $(A_n)_{n=n_0}^\infty$.
Indeed, if infinitely many $A_n$ have a non-trivial automorphism group, so that there is at least one $Q_n\in\Pi_n$ such that $Q_n\neq I_n$ but $A_n=Q_n^TA_nQ_n$, then the MLE has no hope of consistency in the strong or weak sense. In those circumstances, the most we can hope is that the MLE converges to an element within the automorphism group of the graph (see Remark \ref{remark:risk-consistency}). If, on the other hand, the $A_n$'s are sufficiently asymmetric (see Theorem \ref{thm:asym} below for a precise definition of this) then the MLE will be strongly consistent under mild model assumptions.

Before stating the theorem, recall that two permutation matrices $Q,Q'\in\Pi_n$ are disjoint if each one permutes different vertices, that is, if $\sigma_Q,\sigma_{Q'}$ are the permutations associated to each matrix, then $\{i \text{ s.t. }\sigma_{Q}(i)\neq i\}\cap \{i \text{ s.t. }\sigma_{Q'}(i)\neq i\}=\emptyset.$ The same definition holds for sets of $k$ permutations.
To wit we have the following result (the proof is provided in Appendix \ref{sec:pfasym}).

\begin{theorem}
	\label{thm:asym}
	Let $n_0>0$, and consider a sequence of graphs $(A_n)_{n=n_0}^\infty$.  
	Define the sequence $(B_n\sim\text{C}(A_n,p_n,P_n))_{n=n_0}^\infty$ of channel-corrupted $A_n$'s with parameter sequence $(\,(p_n,P_n)\,)_{n=n_0}^\infty$ and $p_n<1/2$.
% 	For each $n\geq n_0$, %and each $k\in[n]$, $k\geq 2$, let $\Pi_{n,k}$ be the set of permutation matrices in $\Pi_n$ permuting precisely $k$ labels, and 
% 	define $\hat{P}_n$ as a maximum likelihood estimator of $P_n$, given by 
% 	$$\hat{P}_n \in \argmin_{P\in\Pi_n}\|A_n - P^TB_nP\|^2_F.$$
	\begin{itemize}
		\item[i.] For each $n\geq n_0$ and each $k\in[n]$, $k\geq 2$, let $\Pi_{n,k}$ be the set of permutation matrices in $\Pi_n$ permuting precisely $k$ labels.
		If there exists an $n_1\geq n_0$ such that for all $n\geq n_1$ and $k\in[n]$,
		\begin{equation}
		\label{eq:growthrate}
		\min_{Q\in\Pi_{n,k}}\|A_n - Q^T A_nQ\|^2_F>\frac{6k\log(n)  }{\log\left(\frac{1}{4p_n(1-p_n)}\right)  }\quad \quad \quad \forall k \geq 2,
		\end{equation}
		then 
		the MLE of $(P_n)_{n=n_0}^\infty$
% 		$(\hat P_n)_{n=n_0}^\infty$ 
		is strongly consistent.
		%\item[ii.] 
		%If there exists a sequence $(k_n)_{n=n_0}^\infty$ such that $k_n=\theta(1)$ and $\epsilon=\epsilon_n=\frac{1}{2}-p_n$, and there exists a $\delta>0$ such that $\epsilon n^\delta=o(1)$, and 
		%\begin{align}
		%\min_{Q_n\in\Pi_{n,k_n}} \|A_n Q_n-Q_n A_n\|^2_F&=\omega(1)\notag\\
		% \label{eq:gr1}
		% \max_{Q_n\in\Pi_{n,k_n}} \|A_n Q_n-Q_n A_n\|^2_F&=o\left(\frac{1}{\epsilon^2}\log\left(\frac{1}{\epsilon\log n}\right)   \right)\\
		%\label{eq:gr2}
		%\max_{Q_n\in\Pi_{n,k_n}} \|A_n Q_n-Q_n A_n\|^2_F&=o\left(\frac{1}{\epsilon^2}\log(n)  \right)
		%\end{align}
		%then any maximum likelihood estimator sequence $(\hat P_n)_{n=n_0}^\infty$ is not weakly consistent. 
		
		\item[ii.] If for infinitey many $n>n_0$ there exists a sequence $(k_n)$ with $k_n=\Theta(1)$ and a set of disjoint permutations $S_{n}\subset (\cup_{i=2}^{k_n}\Pi_{n,i})\setminus\{P_n\}$ with $|S_{n}| = \Theta(n)$ , such that
		\begin{align}
				\label{eq:gr22}
		&\min_{Q_n\in S_{n}}\log\left(\frac{1}{4p_n(1-p_n)}\right)\left(\|A_n -Q_n^T A_nQ_n\|^2_F\right)=\omega(1)\\
	\label{eq:gr23}
&\max_{Q_n\in S_{n}}\log\left(\frac{1}{4p_n(1-p_n)}\right)\frac{\|A_n -Q_n^T A_nQ_n\|^2_F}{\log \|A_n -Q_n^T A_nQ_n\|^2_F}=o(1)
		\end{align}
		then the MLE of $(P_n)_{n=n_0}^{\infty}$ is not weakly consistent. 
		\item[iii.] If for infinitely many $n\geq n_0$, there exists $Q_n\in\Pi_n\setminus\{P_n\}$ such that 
		\begin{equation}\label{eq:thm-condition3}
		 \log\left(\frac{1}{4p_n(1-p_n)}\right)\|A_n - Q_n^T A_nQ_n\|^2_F= O\left(1\right), 
		\end{equation}
		then the MLE of $(P_n)_{n=n_0}^{\infty}$ is not weakly consistent.
	\end{itemize}
	
\end{theorem}

The previous theorem presents necessary and sufficient conditions  in the sequence of graphs $(A_n)$ to determine whether  the MLE for the unshuffling permutation is consistent. Equation~\eqref{eq:growthrate} quantifies how different $A_n$ needs to be from any shuffled version $Q^TA_nQ$ so that the minimizer of the GMP matches the vertices correctly in the limit. Although verifying this condition for a given graph is computationally hard, it can be shown to hold with high probability via a union bound in some popular random graph models, including the Erd\H{os}-R\'enyi model (see Corollaries \ref{cor:examples}, \ref{cor:example2} and \ref{cor:SW}). Note that for a given sequence of graphs $(A_n)_{n=1}^\infty$, Equation \eqref{eq:growthrate} implies that any sequence of corrupting probabilities $(p_n)$ such that
\begin{equation}
p_n \leq \min_{k\in[n]}\left\{p_{n,k}^\ast := \min_{Q\in\Pi_{n,k}} \left[\frac{1}{2} - \frac{1}{2}\left(1 - \exp\left(-\frac{6k\log n}{ \|A_n- Q^T A_nQ\|^2_F}\right)\right)^{1/2}\right]\right\}. \label{eq:upperbound-p_n}
\end{equation}
 will ensure that the MLE is strongly consistent. The right hand side of Equation \eqref{eq:upperbound-p_n} thus quantifies how much noise can be supported by the graph, and hence provides a measure of  matching feasibility.

The second and third parts of Theorem~\ref{thm:asym} provide sufficient conditions for inconsistency of the MLE. If for infinitely many values of $n$ there are  enough permutations $Q_n\in\Pi_n\setminus\{I_n\}$ for which $\|A_n - Q_n^TA_nQ_n\|_F^2$ is small, then there is a non-negligible probability that $\|A_n-Q_n^T(P_nBP_n^T)Q_n\|_F^2 \leq \|A_n-P_nBP_n^T\|_F^2$
for some $Q_n$, and hence, the MLE is not consistent.  The number of permutations required for this to happen depends on how small  $\|A_n - Q_n^TA_nQ_n\|_F^2$ is, and only one permutation is enough if Equation~\eqref{eq:thm-condition3} hold, or $\Theta(n)$ of them with the weaker bound in Equation~\eqref{eq:gr23}. Note that condition \eqref{eq:gr22} can essentially be dropped since when this condition fails then \eqref{eq:upperbound-p_n} holds. Conditions \eqref{eq:gr22}--\eqref{eq:gr23} imply that $p_n\rightarrow 1/2$ and so $\log\left(1/(4p_n(1-p_n))\textbf{}\right)=(1+o(1)) \left(\frac{1}{2}-p_n\right)^2$; this fact will be used often in the sequel  (for example, in Corollary \ref{cor:SW}). Finally, for part ii.,  observe that when $\max_{Q\in S_n}\|A_n-Q_n^TA_nQ_n\|_F^2=\Omega(n^\delta)$ for some $\delta>0$, then the rates in Equations~\eqref{eq:growthrate}  and \eqref{eq:gr23} have the same order.

% This ``critical correlation" appears also as the matchability cutoff in correlated, edge-independent models as well \cite{lyzinski2016information,cullina2016improved}.

\begin{remark} \label{remark:risk-consistency}
    Denote by $\text{Aut}(A_n)=\{Q\in\Pi_n: Q^TAQ = A\}$ the automorphism group of a graph $A_n$. As mentioned before, when $\text{Aut}(A_n)\neq I_n$ for infinitely many values of $n$, there is no hope of consistency, but the estimator might still converge to an element within the automorphism group of the graphs so that $\|\hat P_n^T (P_nA_nP_n^T)\hat{P_n} - A\|_F^2\rightarrow 0$ in some weak or strong sense. In this setting, it is possible to derive a result analogous to Theorem~\ref{thm:asym} by counting the number of edge disagreements $\|A_n-Q^T_nA_nQ_n\|_F^2$ only for the permutations  $Q_n\in\Pi_n\setminus\text{Aut}(A_n)$. The statement and proof of this result are analogous to Theorem~\ref{thm:asym}.
\end{remark}

% \begin{remark}
% 	\label{remark:het-thm}
% 	An equivalent result to  Theorem \ref{thm:asym} for the heterogeneous corrupting channel model also holds. Define a sequence of corrupting matrices $(\Psi^{(1)}_n, \Psi^{(2)}_n)^\infty_{n=n_0}$ with each $\Psi^{(m)}_n$ a $n\times n$ hollow symmetric matrix with entries on $[0,1/2)$, and consider $\tilde p_n = \max_{i>j,\ m} \left(\Psi^{(m)}_n\right)_{i,j}$ and $\breve{p}_n = \min_{i>j,\ m} \left(\Psi^{(m)}_n\right)_{i,j}$. 
% 	Then the MLE in the heterogeneous corrupting channel model is strongly consistent if the conditions in part i. in Theorem \ref{thm:asym} hold, with $\tilde{p_n}$ substituted for $p_n$ in Eq. \eqref{eq:growthrate}. 
% 	Conversely, the MLE is not weakly consistent when part ii. of the previous theorem holds, with $\breve{p}_n$ substituting for $p_n$ in Eq. \ref{eq:gr2}. 
% 	The proof is a straightforward extension to the proof of Theorem \ref{thm:asym}, and hence we omit it.
% \end{remark}

While the correlation structure between $A$ and $B$ in the above Theorem is different than the $\rho$-correlated Bernoulli model of \cite{JMLR:v15:lyzinski14a,lyzinski2016information,rel}, this theorem, in a sense, partially unifies many of the edge-independent matchability results appearing in previous work \citep{JMLR:v15:lyzinski14a,rel,lyzinski2016information,onaran2016optimal,cullina2017exact}.
To wit, we have the following corollary,
% .  Note that Proposition \ref{prop:examples} part (a) 
whose proof follows immediately from Theorem \ref{thm:asym} and results in \cite{rel} (see Appendix \ref{sec:pfexamples} for proof details).
% ; part (b) follows immediately from Theorem \ref{thm:resilient} and results in \cite{kim}.
\begin{corollary}
	\label{cor:examples}
	Let $n_0>0$, and consider an independent sequence of graphs $(A_n)_{n=n_0}^\infty$ where for each $n$, $A_n\sim$Bernoulli$(\Lambda_n)$ with $\alpha_n\leq \Lambda_n\leq 1-\alpha_n$ entry-wise.  
	Define the sequence $(B_n\sim\text{C}(A_n,p_n,P_n))_{n=n_0}^\infty$ of channel-corrupted $A_n$'s with parameter sequence $(\,(p_n,P_n)\,)_{n=n_0}^\infty$ and $p_n<1/2$.
	% \begin{itemize}
	% \item[a.]  
	If $\alpha_n\left(\frac{1}{2}-p_n\right)=\omega\left( \sqrt{\frac{\log n}{n}}\right),$
	then the MLE (obtained by solving the GMP) is strongly consistent.
\end{corollary}

The Erd\H{o}s-Renyi model (ER) \citep{ErdRen1963} $\text{G}(n,\alpha)$, $\alpha\in(0,1)$ is a particular example of a Bernoulli graph $\text{Bernoulli}(\Lambda)$ in which all entries of $\Lambda$ are equal to $\alpha$, and hence the result of the previous corollary directly applies. By considering concentration bounds on the number of edges in Erd\H{o}s-R\'enyi G$(n,\alpha)$ random graphs (i.e., the number of edges is in $n\alpha\pm\sqrt{n\log n}$ with high probability) and the growth rate of $\alpha_n$ in Corollary \ref{cor:examples}, we arrive at the following corollary for Erd\H{o}s-Renyi graphs $\text{G}(n,m)$ (i.e., uniformly distributed on the set of graphs with exactly $m\in\Bbb{N}$ edges): 
\begin{corollary}
	\label{cor:example2}
	% \begin{itemize}
	% \item[a.]  
	Let $n_0>0$, and consider an independent sequence of graphs $(A_n)_{n=n_0}^\infty$ where for each $n$, $A_n\sim$G$(n,m_n)$. %; define $\alpha_n=\frac{m_n}{n}$. 
	Define the sequence $(B_n\sim\text{C}(A_n,p_n,P_n))_{n=n_0}^\infty$ of channel-corrupted $A_n$'s %with parameter sequence $(\,(p_n,P_n)\,)_{n=n_0}^\infty$
	with $p_n<1/2$.
	% \begin{itemize}
	% \item[a.]  
	If $$m_n\left(\frac{1}{2}-p_n\right)=\omega\left( \sqrt{n\log n}\right),$$
	then the MLE of $( P_n)_{n=n_0}^\infty$ is strongly consistent.
	% \item[b.]  If $G\sim$ER$(Q)$ with $\alpha\leq Q\leq 1-\alpha$ entrywise with $\alpha=\omega\left(\frac{\log n}{n}  \right)$, then a.a.s. $G$ is strongly $p$-resilient for $p$ satisfying $\alpha(\frac{1}{2}-p)\sqrt{p}=\omega\left( \sqrt{\frac{\log n}{n}}\right)$.
	% \end{itemize}
\end{corollary}

Small-world graph models aim to reproduce the property that in many real networks the clustering coefficient is relatively small and paths between any pair of vertices have a short length, usually of order $\log n$ \citep{SW_WS}. The Newman-Watts (NW) model \citep{newman1999renormalization}---a more easily analyzed variant of the original Watts-Strogatz model ---is one such model that adheres to these properties with high probability. 
The NW model $\text{NW}(n,d,\beta)$ starts with a $d$-ring lattice $A$ in which each vertex is connected with all neighbors that are at a distance no larger than $d$; that is, $A_{ij}=1$ whenever $0<|i-j| \text{mod}(n-1-d)< d$. To generate a small-world behavior, random edges are added independently with probability $\beta$ between vertices that are not connected in the $d$-ring lattice, so $A_{ij}\sim\text{Bernoulli}(\beta)$ for $|i-j|\ \text{mod}(n-1-d)\geq d$. 

In the Newman-Watts model, the average path length when $\beta=0$ is $O(n/d_n)$, and as $\beta$ increases and more edges are added to the graph, there is a phase transition in which the small-world property appears. The next corollary establishes a similar phenomenon in the matching feasibility context: when $\beta$ is small, the MLE is not weakly consistent for matching a NW graph $A$ with a random copy generated from the corrupting channel model, but once there are enough random edges  included, the MLE becomes strongly consistent and hence graph matching is feasible. See Appendix \ref{sec:pfSW} for the proof details.

\begin{corollary}
	Let $n_0>0$ and consider $(A_n)_{n=n_0}^\infty$ an independent sequence of graphs such that $A_n\sim NW(n, d_n, \beta_n)$. Define $(B_n)_{n=n_0}^\infty$ as a sequence of channel corrupted graphs with $B_n\sim\text{C}(A_n, p_n, P_n)$ and $p_n<1/2$. 
	\begin{enumerate}[a)]
		\item For any sequence $(d_n)_{n=n_0}^\infty$, if $(1/2-p_n)^2=o\left(\sqrt{\log n/n}\right)$ and
		\begin{equation}
		\beta_n=o\left(\frac{\log n}{(1/2-p_n)^2n}\right),\label{eq:SW-unfeasibility}
		\end{equation}
		then the MLE of $(P_n)_{n=n_0}^\infty$ (obtained by solving the GMP~\eqref{eq:GMP}) is not weakly consistent.
		\item If $d_n=o(\beta_n^2 n)$, $\beta_n\leq 1/2$ and 
		\begin{equation}
		\beta_n = \omega\left(\sqrt{\frac{\log n}{\left({1}/{2} - p_n\right)^2n}}\right) \label{eq:SW-feasibility},
		\end{equation}
		then the MLE of $(P_n)_{n=n_0}^\infty$ is strongly consistent.
	\end{enumerate} \label{cor:SW}
\end{corollary}
Using  results on the diameter of Erd\H{o}s-Renyi graphs \citep{chung2001diameter}, it can be verified that if $\beta_n$ satisfies Equation \eqref{eq:SW-unfeasibility} for a fixed value of $p_n<1/2$ and $\beta_n=\omega(1/n)$, then $\text{diam}(A_n) = O\left(\frac{\log (n)}{\log(n\beta_n)}\right)$. Thus, there is a regime for $\beta_n$ in which NW graphs are small-world but the MLE is not consistent.  However, since $\omega(1) = \log(n\beta_n) = o(\log\log n)$ in this regime, the diameter of the graph  is close to $O(\log n)$ and hence it is not far from the phase transition of small-world graphs. If part b) of  Corollary \ref{cor:example2} holds, then $\text{diam}(A_n) = O(1)$, so the average path length in NW graphs with a strongly consistent MLE is much smaller than  in small-world graphs.

Corollary \ref{cor:SW} also suggests a way of making any given graph $A$, with a sufficiently large number of vertices $n$, matchable with a graph corrupted with probability $p$. 
By setting $\beta=\sqrt{\log(n)/((1/2-p)^2n)}$ as in Equation \eqref{eq:SW-feasibility} and generating a graph $A'\sim\text{C}(A, \beta, I)$, the solution of the graph matching problem between $A'$ and $B$ will correctly align all of the vertices with high probability. 
This result can have applications in settings such as network anonymization \citep{narayanan2009anonymizing} and differential privacy \citep{dwork2014algorithmic}, among others.
%\JA{are there any problems in which a result like this might be useful?}

\subsection{Consistency on the heterogeneous corrupting channel}

We now present a more general version of Theorem~\ref{thm:asym} for the heterogeneous corrupting channel introduced in Section~\ref{sec:hetchannel}. In the heterogeneous model, $\Psi^{(1)}$ and $\Psi^{(2)}$ are nuisance parameters when estimating the permutation parameter $P$, so the number of nuisance parameters is $n(n-1)$, which scales with the sample size. Moreover, as can be observed from Equation~\eqref{eq:nuisance}, the MLE of these parameters is not consistent.  In this setting, one should be careful when using the MLE for $P$ since this estimator can be  inconsistent as shown in the famous Neyman-Scott paradox \citep{neyman1948consistent}. Fortunately, this is not the case here and the MLE can consistently estimate $P$.  See Section \ref{sec:pfasym} for the proof of this Theorem.

\begin{theorem}
	\label{thm:asym-het}
	Let $n_0>0$, and consider a sequence of graphs $(A_n)_{n=n_0}^\infty$.  
	Define the sequence $(B_n\sim\text{C}(A_n,\Psi^{(1)}_n, \Psi^{(2)}_n,P_n))_{n=n_0}^\infty$ of channel-corrupted $A_n$'s with parameter sequence $(\,(\Psi_n^{(1)},\Psi_n^{(2)},P_n)\,)_{n=n_0}^\infty$ and $\tilde{p}_n:=\max_{i,j,k}(\Psi^{(k)}_n)_{ij}<1/2$, $\breve{p}_n=\min_{i,j,k}(\Psi^{(k)}_n)_{ij}\geq 0$.

	\begin{itemize}
		\item[i.] For each $n\geq n_0$ and each $k\in[n]$, $k\geq 2$, let $\Pi_{n,k}$ be the set of permutation matrices in $\Pi_n$ permuting precisely $k$ labels.
		If there exists an $n_1\geq n_0$ such that for all $n\geq n_1$ and $k\in[n]$,
		\begin{equation}
		\label{eq:growthrate-het}
		\min_{Q\in\Pi_{n,k}}\|A_n - Q^T A_nQ\|^2_F>\frac{6k\log(n)  }{\log\left(\frac{1}{4\tilde{p}_n(1-\tilde{p}_n)}\right)  }\quad \quad \quad \forall k \geq 2,
		\end{equation}
		then 
		the MLE of $(P_n)_{n=n_0}^\infty$ is strongly consistent.

		\item[ii.] If for infinitely many $n>n_0$ there exists a sequence $(k_n)$ with $k_n=\Theta(1)$ and a set of disjoint permutations $S_{n}\subset (\cup_{i=2}^{k_n}\Pi_{n,i})\setminus\{P_n\}$ with $|S_{n}| = \Theta(n)$, such that 
			\begin{align}
				\label{eq:gr22-het}
		&\min_{Q_n\in S_{n}}\log\left(\frac{1}{4\tilde p_n(1-\tilde p_n)}\right)\left(\|A_n -Q_n^T A_nQ_n\|^2_F\right)=\omega(1)\\
	\label{eq:gr23-het}
&\max_{Q_n\in S_{n}}\log\left(\frac{1}{4\breve p_n(1-\breve  p_n)}\right)\frac{\|A_n -Q_n^T A_nQ_n\|^2_F}{\log \|A_n -Q_n^T A_nQ_n\|^2_F}=o(1)
		\end{align}
		then the MLE of $(P_n)_{n=n_0}^{\infty}$ is not weakly consistent. 
		\item[iii.] If for infinitely many $n\geq n_0$, there exists $Q_n\in\Pi_n\setminus\{P_n\}$ such that 
		\begin{equation}\label{eq:thm-condition3-2}
		 \log\left(\frac{1}{4\breve p_n(1-\breve  p_n)}\right)\|A_n - Q_n^T A_nQ_n\|^2_F= O\left(1\right), 
		\end{equation}
		then the MLE of $(P_n)_{n=n_0}^{\infty}$ is not weakly consistent.
	\end{itemize}
\end{theorem}

As described in Section~\ref{sec:hetchannel}, the heterogeneous corrupting channel model can describe pairs of Bernoulli graphs with correlated edges. In particular, the correlated Erd\H{o}s-R\'enyi graph model with parameters $q,\rho\in[0,1]$, denoted as $\rho$-ER$(q)$ generates  a pair of networks $A,B$ that are marginally distributed as Erd\H{o}s-R\'enyi $A,B\sim\text{G}(n,q)$, but  each pair of corresponding edges satisfies $\text{corr}(A_{ij},B_{ij})=\rho$, while the rest of the variables are mutually independent. The following result shows the consistency of the MLE in this model.

\begin{corollary}
	Let $n_0>0$, and consider an independent sequence of graphs $(A_n,B_n)_{n=n_0}^\infty$ where for each $n$, $(A_n,B_n)\sim$ $\rho_n$-ER$(q_n)$, with $q_n\leq 1/2$ and $\tilde{p}_n:=(1-q_n)(1-\rho_n)\leq 1/2$. 
	If $q_n\log\left(\frac{1}{4\tilde{p}_n(1-\tilde{p}_n)}\right)=\omega\left( {\frac{\log n}{n}}\right)$ and $q_n=\omega\left( {\frac{\log n}{n}}\right)$,
	then the MLE (obtained by solving the GMP) is strongly consistent. \label{cor:correlated}
\end{corollary}

The correlated Erd\H{o}s-R\'enyi model has been popularly used to study the performance of different methods in solving the graph matching problem \citep{JMLR:v15:lyzinski14a,cullina2017exact,ding2018efficient,barak2019nearly}, and in particular, \cite{cullina2017exact} established the information-theoretic threshold  for exact recovery of the unshuffling permutation. The parameterization of the corrupting channel model and the constraint on the parameters to obtain the equivalence between MLE and the GMP enforces some additional restrictions in our setting, but we can verify than in some parameter regimes the rates of Corollary~\ref{cor:correlated} are optimal. For example, when the graphs are dense and $q_n=\frac{1}{2}$, the conditions of Corollary~\ref{cor:correlated} reduce to $\log\left(\frac{1}{4\tilde{p}_n(1-\tilde{p}_n)}\right)\sim \rho_n=\omega(\sqrt{\frac{\log n}{n}})$, achieving the information theoretic rate \cite{cullina2017exact,barak2019nearly}. Similarly, when $\rho_n=\Omega(1)$, the requirement on the density of the graphs is $q_n=\omega\left(\frac{\log n}{n}\right)$.

\begin{remark}
    A further generalization of the heterogeneous corrupting channel model can consider dependent noise, in which the random variables $\{X_{i,j}, Y_{i,j}\}_{i<j}$ introduced in Definition~\ref{def:corrupted-het} are dependent.
    Concentration results for sums of dependent random variables can be used to obtain  an analogous statement to part i. of Theorem~9 for a suitable dependence structure between the edges \citep{janson2004large,liu2019mcdiarmid}.
    
\end{remark}

\section{Relaxed consistency}
\label{sec:quasi}

In many label recovery settings in the network literature, consistency is defined as recovering the correct labels of all but a small, vanishing fraction of the vertices; for examples in the community detection literature, see \cite{rohe2011spectral,sussman2014consistent,fishkind2013consistent,qin2013dcsbm}, in the graph matching literature, see \cite{zhang2018consistent,zhang2018unseeded}.
In real data settings, networks are often complex, heterogeneous objects, and defining consistency as perfect recovery of the alignment/labels is an often unrealistic standard to apply in practice.
This is especially the case in the presence of network sparsity, as often
recovering the alignment/labels of low degree vertices is practically and theoretically impossible.

In our errorful channel setting it is entirely reasonable to expect low-degree vertices in $A$ to have their labels irrevocably permuted by the channel permutation $P$.
We demonstrate this in the following illustrative example.
\begin{example}
	\label{ex:leaves}
	\emph{
		Consider $A_n$ constructed as follows.  Let $G_n\sim$G$(n-2,q)$ be an ER graph for fixed $q$.  $A_n$ is then formed by 
		attaching two leaves (labeled $n-1$ and $n$) to $G_n$ at two uniformly randomly chosen elements of $G_n$.
		Form $B_n$ via the channel parameters $(p_n,P_n)$ with $p_n=p$ fixed and $P_n$ the transposition of $n-1$ and $n$. 
		It is immediate that with probability at least $p^2$ (i.e., if the two edges connecting the leaves to $G_n$ are corrupted), the MLE will not equal $P_n$.}
\end{example}

In the previous example, we see that the MLE will often be unable to recover periphery vertex labels. 
However, in Example \ref{ex:leaves} it is not difficult to see that for $p$ fixed, the MLE will recover all but potentially the two peripheral vertex labels.
It will be practically useful to extend the definition of consistency to account for these situations where the MLE can recover almost all of the vertex labels.
This motivates our next definition.
\begin{definition}
	\label{def:anconsis}
	%With notation as above, for $n>n_0$ define $$\hpi\in\text{argmin}_{P\in\Pi_{n}}\|A_n-P^TB_nP\|_F.$$
	For a sequence of numbers $(\alpha_n)_{n=n_0}^{\infty}$,
	we say that a sequence of estimators $(\hat P_n:=\hat P_n(A_n,B_n))_{n=n_0}^\infty$ of $(P_n)_{n=n_0}^\infty$,
	is a $\begin{pmatrix}\text{weakly}\\\text{strongly} \end{pmatrix}$ $\alpha_n$-consistent estimator of $(P_n)_{n=n_0}^\infty$ if $$\frac{\|Q-P_n\|_F}{\alpha_n}\begin{pmatrix}\stackrel{P }{\rightarrow}\\ \stackrel{ a.s.}{\rightarrow} \end{pmatrix}0$$ as $n\rightarrow\infty$ for all sequences $(P_n)_{n=n_0}^\infty$.
\end{definition}
\noindent As before, as the MLE of $(P_n)_{n=n_0}^\infty$ is not necessarily uniquely defined (i.e., there are multiple possible MLEs for a given $n$), we write that the \emph{MLE of $(P_n)_{n=n_0}^\infty$ is 
$\begin{pmatrix}\text{weakly}\\\text{strongly} \end{pmatrix}$ $\alpha_n$-consistent} if $$\frac{\max_{Q_n\in \mathcal{P}^*_n}\|Q_n-P_n\|_F}{\alpha_n}\begin{pmatrix}\stackrel{P }{\rightarrow}\\ \stackrel{ a.s.}{\rightarrow} \end{pmatrix}0$$ as $n\rightarrow\infty$,
where
$$\mathcal{P}^*_n=\argmin_{P\in\Pi_n}\|A_n - P^TB_nP\|^2_F.$$
% JA: For both consistency and relaxed consitency definitions, I made them more general (before they were defined for the MLE only), and I moved the definition of the MLE to the Theorem statements

This relaxed definition allows us to consider situations in which the MLE will correctly align all but a vanishing fraction of the vertices in $(A_n,B_n)$; indeed, if $\alpha_n=o(n)$, then the fraction of misaligned vertices is $o(1)$. Note also that a consistent estimator according to Definition~\ref{def:consis} is $\Theta(1)$-consistent, and viceversa.
In the remainder of this section, we will establish $\alpha_n$ consistency for a variety of models.
\vspace{3mm}

\noindent{\bf Example \ref{ex:leaves} continued.}  Consider $A_n$ as constructed in Example \ref{ex:leaves}.  If we consider $p_n\equiv p<1/2$ fixed, then it is immediate that the MLE will be strongly $\alpha_n$-consistent for any $\alpha_n=\omega(1)$.
Indeed, with high probability the only two vertices potentially misaligned by the MLE for $n$ sufficiently large are the two leaves attached to $G_n$. 
\vspace{3mm}

The notion of $\alpha_n$-relaxed consistency allows to have at most $o(\alpha_n)$ vertices incorrectly matched, and hence, to check whether an estimator is consistent it is enough to only consider the permutations that permute at least $\Omega(\alpha_n)$ vertices. Hence, the $\alpha_n$-consistency analogue of Theorem \ref{thm:asym}, parts i. and iii., can be formulated as follows.
The proof is omitted, as it follows the proof of Theorem \ref{thm:asym} mutatis mutandis by only considering permutations of enough vertices.

\begin{theorem}
	\label{thm:asym2}
	Let $n_0>0$, and consider a sequence of graphs $(A_n)_{n=n_0}^\infty$.  
	As above, define $(B_n)_{n=n_0}^\infty$ to be the sequence of channel-corrupted $A_n$'s with parameter sequence $(\,(p_n,P_n)\,)_{n=n_0}^\infty$ and $p_n<1/2$. If we have that there exists an $n_1\geq n_0$ such that for all $n\geq n_1$
	\begin{equation}
	\label{eq:growthrate-quasi}
	\min_{Q\in\Pi_{n,k}}\|A_n -Q^T A_nQ\|^2_F\geq\frac{6k\log(n)  }{ \log\left(\frac{1}{4p_n(1-p_n)}\right)  }\quad \quad \quad \forall k \geq k_n,
	\end{equation}
	then the MLE of $(P_n)_{n=n_0}^\infty$ is strongly $\alpha_n$-consistent for any sequence $(\alpha_n)_{n=n_0}^{\infty}$ with $\alpha_n=\omega(k_n)$.

\end{theorem}

%\begin{proof}
%The proof of the first part follows the proof of Theorem \ref{thm:asym} mutatis mutandis.

%For the second part, define $\Pi_{n}(\mathcal{M}_n) = \{Q\in\Pi_n| (Q_n)_{i,i}=1\text{ for } i \in\mathcal{M}_n\}$ to be the set of permutations that keep the vertices on $\mathcal{M}_n$ unchanged. It is enough to show that for any $Q\notin\Pi_n(\mathcal{M}_n)$, there is a high chance that some permutation $Q'_n \in\Pi_n(\mathcal{M}_n)$ satisfies
%\begin{equation*}
%    \|A_nQ_n - Q_nA_n\|^2_F > \|A_nQ'_n - Q'_nA_n\|_F^2.
%\end{equation*}

%\end{proof}

The  previous theorem is analogous to part i. of Theorem~\ref{thm:asym}, but here Equation~\eqref{eq:growthrate-quasi} is only required to hold for permutations that shuffle at least $k_n$ vertices, so if $k_n=\omega(1)$ then condition a) is significantly weaker. %Likewise, part b) is analogous to part iii. in Theorem~\ref{thm:asym} with a permutation that shuffles at least $\Omega(\alpha_n)$ vertices.  
Part ii. of Theorem \ref{thm:asym} requires the existence of a set with $\Theta(n)$ disjoint permutations, so most permutations in this set need to be permuting $\Theta(1)$ vertices, hence this statement is capturing the setting when the MLE with high probability will not align a small ($\Theta(1)$ in the Theorem) number of vertices. As such, it is not immediate what the $\alpha_n$-consistency analogue of this result would be.

Our next result partially extends Corollaries \ref{cor:examples} and \ref{cor:example2} to the case of random regular graphs. Let $A\sim$G$_{n,d}$ be shorthand for $A$ is uniformly distributed on the set of $d$-regular, $n$-vertex graphs. Solving the exact version of the graph matching problem (that is, for $p=0$ in the corrupting channel model) is usually (at least theoretically) possible, as these graphs almost surely have a trivial automorphism group under mild assumptions \citep{kim}. 
The next corollary thus partially extends this result to the inexact graph matching setting (proof provided in Appendix \ref{sec:pfexample3}). 
% \JA{Is this result new? If so, framing it as a new result on inexact matching for $d$ regular graphs might be worth to highlight in the intro.}

\begin{corollary}
	\label{cor:example3}
	Let $\epsilon\in(0,1/3)$ be fixed, and consider $d_n=\omega(n^{2/3+\epsilon})$ (with $d_n$ bounded away from $n$). 
	There exists a constant $C>0$ such that 
	if $A\sim$G$_{n,d_n}$ and $1/2 - p_n \geq\sqrt{Cd_n^{-1}{\log n}}$,
	then the MLE of  $(P_n)_{n=n_0}^\infty$ is strongly $\alpha_n$-consistent for $\alpha_n=\omega(n^{2/3})$.
\end{corollary}

The previous result raises the question of whether there will always be a fraction of misaligned node in a random regular graph. Although we do not have a definite answer, it is possible to  construct non-trivial examples of  sequences of $d$-regular graphs for which the MLE is not strongly consistent according to Definition \ref{def:consis}. In fact, consider $A_n$ as the $d_n$-ring lattice with $n$ vertices and any sequence of $(d_n)$. Then, according to parth a) of Corollary \ref{cor:SW}, the MLE is not strongly consistent for any $(P_n)_{n=n_0}^\infty$. Numerical simulations on the $d_n$-ring lattice (see Section \ref{sec:simulations}) suggest that for the $d_n$-ring lattice in particular, a large fraction of the vertices can be correctly aligned if $d_n$ is large enough.

\section{Experiments}
\label{sec:experiments}
In light of the theoretical results presented in the previous sections, we perform experiments to study the matchability of channel corrupted graphs. Using simulated and real data, we study the question of whether the vertices of a given graph are matchable to the vertices of a random graph generated after passing the original graph through a corrupting channel.

As observed in the previous sections, the feasibility of graph matching in the errorful channel model depends on the number of edge disagreements between the original graph and a shuffled version of it. 
Equation \eqref{eq:upperbound-p_n} provides an upper bound on the probability $p_n$ of the corrupting channel model that ensures that the MLE is a consistent estimator;
indeed, $\min_{k} p^\ast_{n,k}$ gives a measure of how much noise can a graph $A$ safely tolerate while keeping the feasibility of perfect graph de-anonymization in the limit. 
Similarly, using the result in Theorem~\ref{thm:asym2}, $\min_{k\geq k_0} p_{n,k}^\ast$ gives a (limiting) measure of how much noise can a graph $A$ safely tolerate while keeping the feasibility of de-anonymizing all but $o(k_0)$ of the vertices. 
The value of $p^\ast_{n,k}$ depends on $X_{k,\min}(A) =\min_{Q\in\Pi_{n,k}}\|A-Q^TAQ\|_F^2$, the smallest number of edge disagreements over all permutations of $k$ vertices.  Computing the exact value of $p^\ast_{n,k}$ is infeasible for large values of $k$, but using a sample of random permutations we can obtain an exact upper bound for $p^\ast_{n,k}$, given by
\begin{equation}
\hat{p}^\ast_{n,k}(A) = \frac{1}{2} - \frac{1}{2}\left(1-\exp\left(-\frac{6k\log n}{\hat{X}_{k,\min}(A)}\right)\right)^{1/2}, \label{eq:hat-upperbound-p_n}
\end{equation}
where $\hat{X}_{k,\min(A)}=\min_{i=1,\ldots,m}\|A-Q_i^TAQ_i\|_F^2$ for a set of random permutations $Q_1,\ldots,Q_m$, uniformly distributed on $\Pi_{n,k}$. Note that this bound $\hat{p}_{n,k}^\ast$ might be loose compared with the one in Equation~\ref{eq:upperbound-p_n}, as only one permutation that was not sampled can decrease $X_{k,\min}$ significantly. Nevertheless, as observed in Section~\ref{sec:consis},  in many random graph models (including edge-independent graphs), the distribution of $\|A_n-Q^TA_nQ\|_F^2$ is usually bounded away from zero.

In addition, motivated by the bounds in Theorems \ref{thm:asym} and \ref{thm:asym2}, we estimate the average value of $\|A-QAQ^T\|_F^2$ where $Q\sim \text{Unif}(\Pi_{n,k})$ by
counting the number of edge disagreements introduced in $A$ by a uniform shuffling of $k$ vertices, and calculate
\begin{equation}
\widehat{X}_k(A) = \frac{1}{m}\sum_{i=1}^m \frac{1}{2}\|A-Q_i^TAQ_i\|_F^2. \label{eq:X_hat}
\end{equation}
In both cases, we generate a sample of $m=1000$ permutations.
In practice, we have found that larger values of $\widehat{X}_k(A)$ indicate that graph matching (at least recovering all but $k$ labels) is more feasible. 
We use a normalized version of this quantity $\frac{1}{k\log n}\hat{X}_k(A)$ to compare this number between different permutation sizes and number of vertices, where the normalizing constant is motivated by Theorem \ref{thm:asym}. 

These two statistics, $\hat{p}_k^\ast(A)$ and $\hat{X}_k(A)$, are calculated for different values of $k$, 
and they are used  to compare matching feasibility in a variety of graphs generated from popular statistical models and real networks from different domains. In addition, we compare the results of these statistics with the  accuracy of  a matching algorithm for a given graph $A$ and a corrupted version of it. Specifically, for a given graph $A$ we generate a corrupted graph $B\sim\text{C}(A, p, P)$ using the channel model, and approximately solve the resulting GMP \eqref{eq:GMP} to estimate the MLE (as described below). We measure the number of incorrectly matched vertices $\|\hat{P}-P\|_F^2$ for different values of $p$, and compare these results with the feasibility measures previously obtained.

In practice, computing the MLE of the unshuffling permutation is computationally unfeasible in general since solving the GMP requires to perform an exhaustive search  over the discrete space of $n!$ permutations; provably polynomial time algorithms  are only known for certain models such as correlated Erd\H{os}-R\'enyi graphs \citep{barak2019nearly,ding2018efficient}. We obtain an approximate solution to the GMP  with the Fast Approximate Quadratic programming algorithm  \citep{VogConLyzPodKraHarFisVogPri2015,rel}, which uses the Frank-Wolfe methodology \citep{FraWol2016} to solve an indefinite relaxation of the graph matching objective function before projecting the relaxed solution to the feasible space of permutations.
Frank-Wolfe is a constrained gradient descent procedure, and the algorithm terminates in a permutation $\hat{P}^{\text{FAQ}}$ that is an estimated local minimum of the relaxed GMP objective function. The computational complexity of this method is $O(n^3)$, but because it is solving a non-convex optimization problem, there is no guarantee that the global minimum is achieved. When the indefinite relaxation problem is solved exactly, it recovers the correct permutation almost always in correlated Bernoulli graphs \citep{rel}.

\subsection{Simulations}
\label{sec:simulations}
In the first example, we generate graphs from two popular graph models and compare the effect of the parameters that control the structure of the network. First, we simulate Erd\H{o}s-Renyi graphs  $\text{G}(n,\alpha)$ with $n=500$ vertices and rate $\alpha$, changing the value of $\alpha\in\{0.1,0.2, 0.3, 0.4, 0.5\}$ to compare the effect of the average degree. 
We also generate graphs using the Watts-Strogatz small-world model $\text{WS}(n, d, \beta)$, which are very similar to the Newman-Watts model (see Corollary \ref{cor:SW}); these graphs are intended  to produce graphs with the \textit{small-world} property. 
The WS model is  initialized with a regular $\frac{d}{2}$-ring lattice like the NW model, and each edge is randomly rewired with probability $\beta$. As $\beta$ increases, the distribution of the WS model becomes more similar to an ER graph. 

For each graph, we compute the estimated probability bounds $\hat{p}^\ast_k(A)$ and the normalized average disagreements $\frac{1}{k\log n}\hat{X}(A)$ according to equations \eqref{eq:X_hat} and \eqref{eq:hat-upperbound-p_n}. We generate 35 random graphs according to each model and compute the average quantities.
These results are summarized on Figure \ref{fig:experiment-ER}. 
We can see that in the ER model channel probability bounds and the average number of disagreements increase  with  $\alpha$, which suggests that matching in the corrupting channel model becomes more feasible as the average degree of the graph increases, and verifies the results in Corollary \ref{cor:examples}. Graphs with a WS distribution increase their matchability measures as the rewiring probability increases in accordance to Corollary \ref{cor:SW}, and this suggests that graphs that have a more uniformly random structure are easier to match.  
Matching all the vertices correctly in the WS model is difficult when $\beta$ is small since the near degree regularity ensures that vertices are very similar to their neighbors, and hence switching a few of the vertices with their neighbors causes a small number of disagreements. 
This result agrees with Corollary \ref{cor:SW} for the related Newman-Watts small-world model. Nevertheless, Figure \ref{fig:experiment-WS} also shows that even for small $\beta$, permuting a large fraction of vertices also causes this probability bound to increase significantly, and by Theorem~\ref{thm:asym2} this suggests that it might be still feasible to match a large portion of the vertices correctly. 

\begin{figure}
	\centering
	\begin{subfigure}{0.48\textwidth}
		\includegraphics[width=\textwidth]{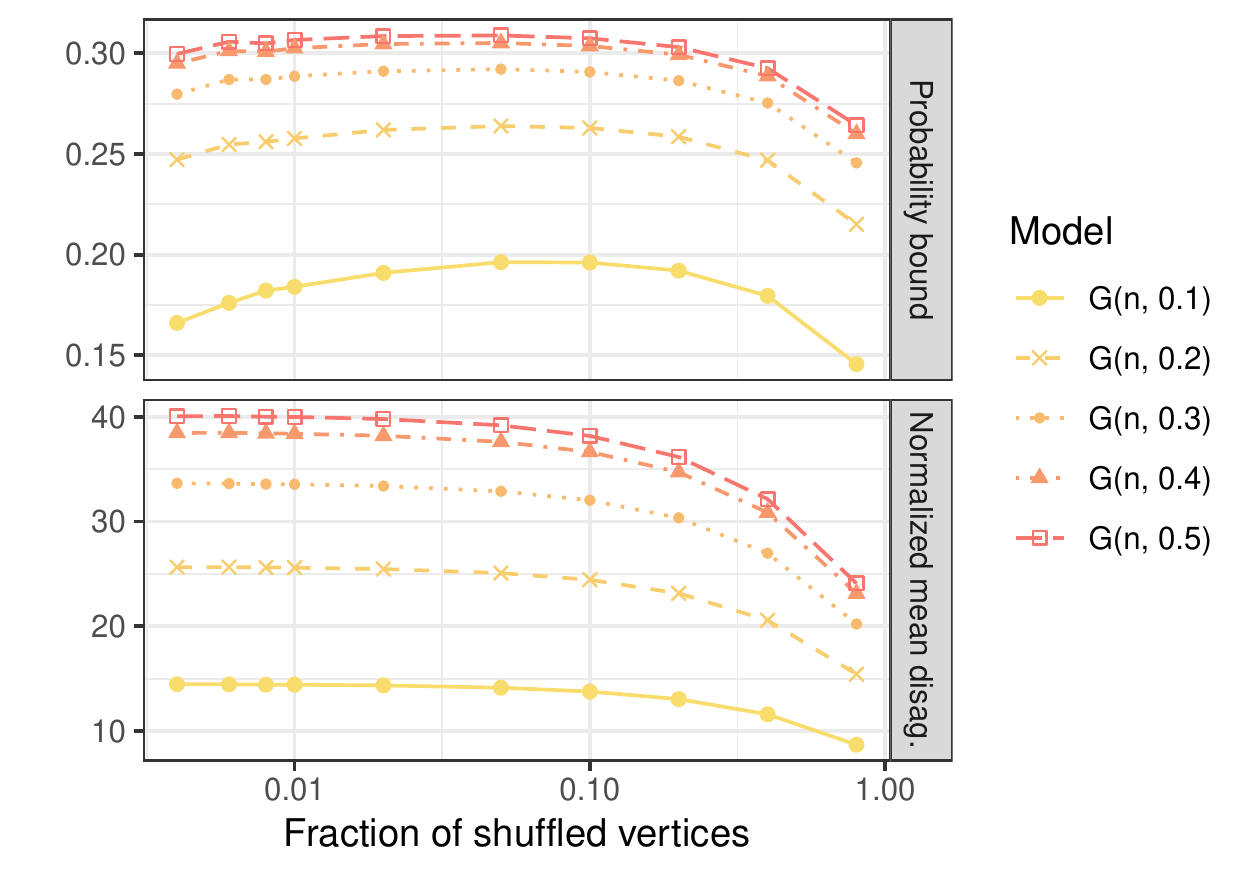}
		\caption{Erd\H{o}s-R\'enyi G($n,\alpha$)}
		\label{fig:experiment-ER}
	\end{subfigure}
	\begin{subfigure}{0.48\textwidth}
		\includegraphics[width=\textwidth]{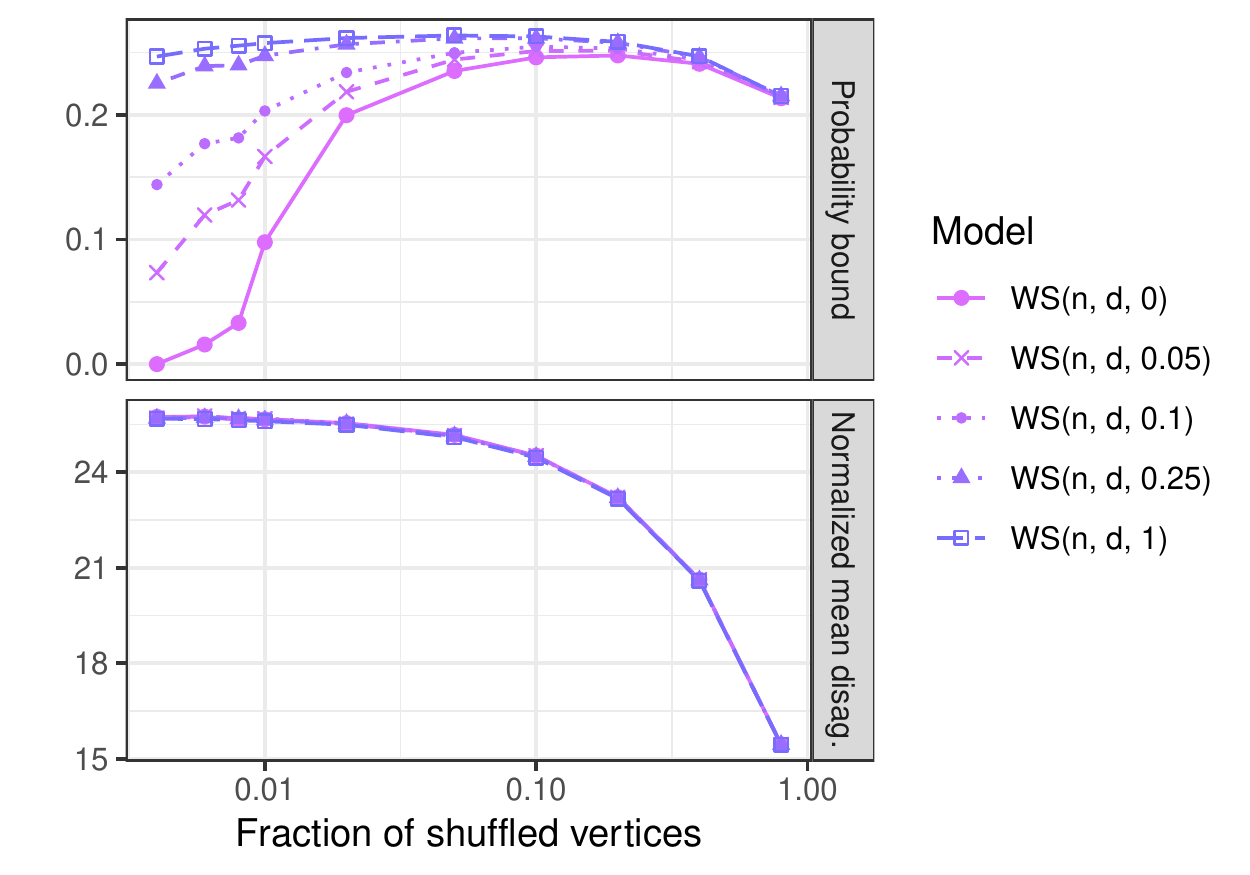}
		\caption{Watts-Strogatz WS($n,d,\beta$)}
		\label{fig:experiment-WS}
	\end{subfigure}
	\begin{subfigure}{0.8\textwidth}
		\centering
		\includegraphics[width=0.7\textwidth]{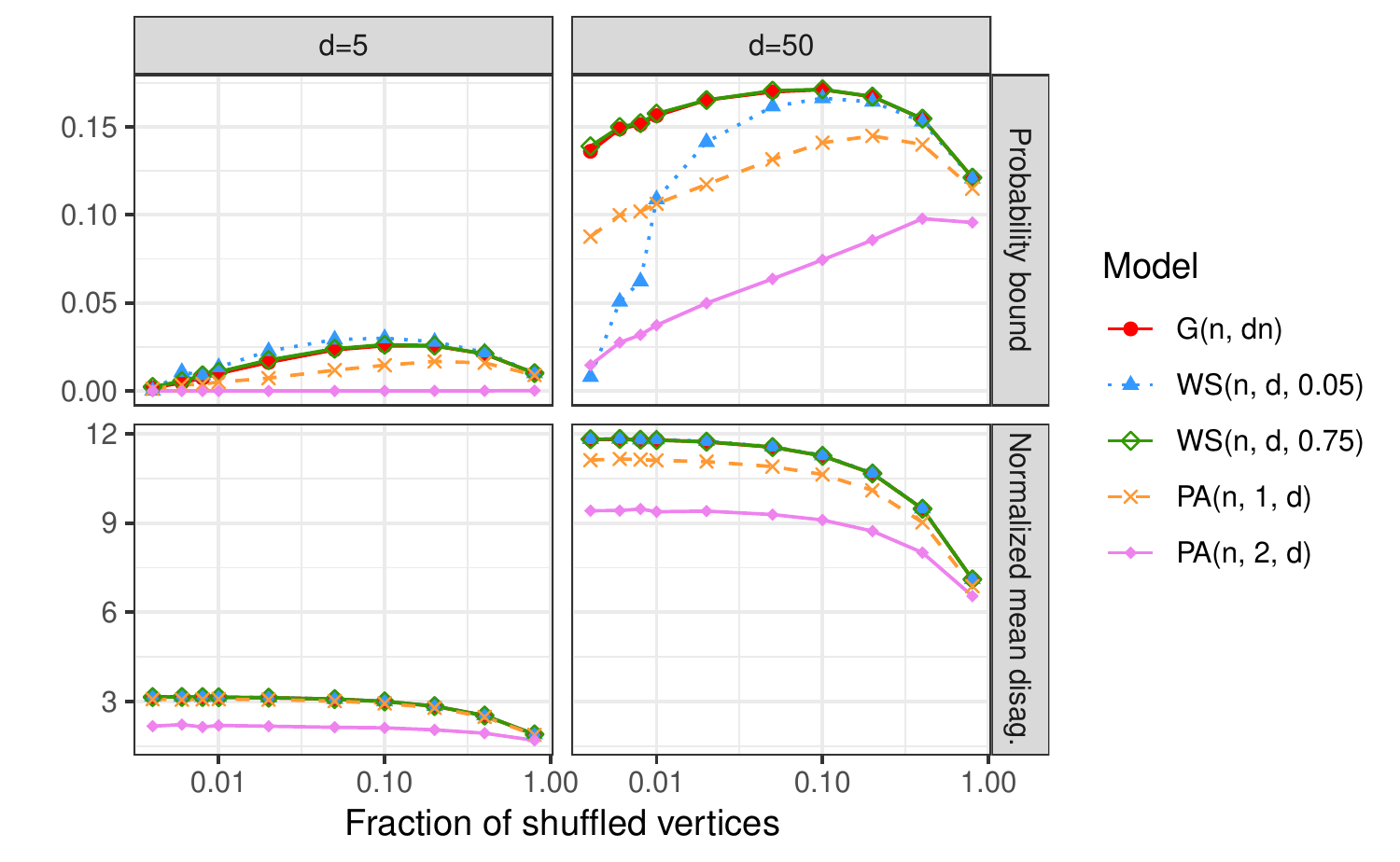}
		\caption{Comparison of different random graph models.}
		\label{fig:experiment-models}
	\end{subfigure}
	%\caption{Measures of matching feasibility for Erd\H{o}s-Renyi graphs G($n,\alpha$)  (left) and Watts-Strogatz graphs WS($n,d,\beta$) (right) as a function of the fraction of shuffled vertices $k/n$. The top panel shows an upper bound on the corrupting probability $p$ supported by the channel acording to Theorem \ref{thm:asym}, and the bottom panel shows the normalized mean edge disagreements after shuffling $k$ random vertices of the graph. Graph matching becomes more feasible as the graphs become denser in the ER model, and less structured in the WS model.}
	\caption{Measures of matching feasibility for different random graph models as a function of the fraction of randomly shuffled vertices $k/n$. The top left figure compares Erd\H{o}s-Renyi graphs G($n,\alpha$) with different edge probability $\alpha$. The top right figure shows the Watts-Strogatz model WS($n,d,\beta$) with different rewiring probabilities $\beta$. The bottom figure compares these two models and the preferential attachment model  with exponent $\gamma$ (PA($\gamma,d$)) for a fixed degree $d$. In each figure, the top panel shows an upper bound on the corrupting probability $p$ tolerated by the graph, and the bottom panel shows the normalized mean edge disagreements after shuffling $k$ random vertices. The plots suggest that graph matching is more feasible in denser and less structured graphs similar to ER.}
\end{figure}

In the second experiment, we compare the same statistics across different popular random graph models, fixing the expected average degree across the graphs. 
In particular, we use again the ER model G($n,m$) but now with $m$ fixed edges and the WS model. 
We additionally include the preferential attachment PA($n,\gamma$,d) model proposed by \citep{PA}.  This model creates a graph by a random process in which a new vertex with $d$ edges is added to the graph on each step $t$. 
For each new vertex the probability that it connects to an existing vertex $i$ is proportional to $d_{it}^{-\gamma}$, where $d_{it}$ is the degree of vertex $i$ at time $t$, and $\gamma\geq 0$ is a constant that controls the preferential effect. 
Larger values of $\gamma$ increase the preference of new vertices to connect with high degree vertices.

As observed in the previous experiment, the edge density of the graph plays an important role for matching feasibility. 
Thus, to make fair comparisons between the different models we adjust the model parameters to generate graphs with the same average degree $d$, by fixing  $d=5$ (1\% density) and $d=50$ (10\% density). 
For the WS model, we generate graphs from a $\text{WS}(n, d,0.05)$ and $\text{WS}(n, d,0.75)$, and for the PA model, we change the exponent $\gamma$ to generate $\text{PA}(n, 1,d)$ and $\text{PA}(n,2,d)$. 
In all cases, we use the default implementation of \texttt{igraph} \citep{csardi2006igraph} to simulate the graphs. 
The results of these experiments are shown in Figure \ref{fig:experiment-models}. We observe that in general the graphs that have a more random structure (the ER model and the WS model with a large rewiring probability) are the ones in which the measures of matching feasibility are larger. 
Matching in the PA model is complicated due to the low degree of many vertices, and thus the theoretical probability bound for perfect matchability is low.

%\begin{figure}
%	\centering
%	\includegraphics[width=0.7\textwidth]{simulations-models.pdf}
%	\caption{Measures of matching feasibility for popular random graph models as a function of the fraction of shuffled vertices $k/n$, and fixing the  average degree $d$ on each column. In all cases, $n=500$. The models considered are the ER with  $dn$ edges (G$(n,dn)$), Watts-Strogatz with low and high rewiring probability $\beta$ (WS($n,\beta$)), and preferential attachment with different exponent $\gamma$ (PA($\gamma,d$)). Matching is more feasible in graphs that are less structured, and hence  more similar to an ER graph. }
%	\label{fig:experiment-models}
%\end{figure}

\subsection{Real-world networks}
We also analyze graph matching in the corrupting channel model for real-world networks from different domains. 
The networks that we use  are the Zachary's karate club friendship network \citep{Zachary1977}, the graph of synapses between neurons of the C. elegans roundworm \citep{SW_WS}, the graph of hyperlinks between political blogs from the 2004 US election \citep{adamic05:_how}, a protein-protein interaction network in yeast \citep{von2002comparative}, and a citation network between arXiv papers in the condensed matter section \citep{newman2001structure}. 
Some graph statistics to summarize the data are included in Table \ref{table:summary-stats}. These include the number of nodes $n$, the average node degree $d=\frac{1}{n}\sum_{ij} A_{ij}$, the density of the graph $ d/(n-1)$, the clustering coefficient $C$ which counts the number of triangles in the graph divided over the maximum number of triangles possible, the skewness $\gamma_1$ and the relative standard deviation (RSD) of the degree distribution. In general, as observed in the simulations, we should expect that as the graphs become denser and with a more random structure (lower clustering coefficient and homogeneous degrees), matching becomes more feasible. 

% latex table generated in R 3.5.1 by xtable 1.8-3 package
% Mon Nov  5 12:43:04 2018
\begin{table}[ht]
	\centering
	{\scriptsize
		\begin{tabular}{l|rrrrrr}
			\hline
			Network & $n$ & $d$ & Density & $C$ & $\gamma_1$ & RSD\\ 
			\hline
			%Karate & 34 & 4.588 & 0.139 & 0.256 & 2.001 \\ 
			%C. elegans & 297 & 15.791 & 0.053 & 0.181 & 3.505 \\ 
			%Pol. blogs & 1222 & 27.355 & 0.022 & 0.226 & 3.063 \\ 
			%Prot. interaction & 2617 & 9.060 & 0.003 & 0.469 & 3.958 \\ 
			%arXiv & 16726 & 6.691 & 0.0004 & 0.359 & 4.068\\
			Karate & 34 & 4.588 & 0.14 & 0.25 & 2.00 & 0.84 \\ 
			C. elegans & 297 & 15.79 & 0.05 & 0.181 & 3.50 & 0.88 \\ 
			Pol. blogs & 1,222 & 27.35 & 0.02 & 0.226 & 3.06 & 1.4 \\ 
			Prot. interaction & 2,617 & 9.06 & 0.003 & 0.47 & 3.96 & 1.65 \\ 
			arXiv & 16,726 & 6.69 & 0.0004 & 0.36 & 4.07 & 0.96 \\ 
			\hline
	\end{tabular}}
	\caption{Summary statistics of the network data: number of vertices ($n$), average vertex degree ($d$), density of the graph , clustering coefficient ($C$), skewness ($\gamma_1$) and relative standard deviation (RSD) of degree distribution . \label{table:summary-stats}}
\end{table}

Figure \ref{fig:exp-real-data} shows the bounds in the channel probability $\hat{p}_k(A)$ (top panel) and the normalized average number of edge disagreements $\hat{X}_k(A)$ (bottom panel) for the selected networks. 
When the fraction of shuffled vertices is small, all networks have a zero tolerance for noise in the corrupting channel model,  which can be due to the periphery nodes and the existence of (near) graph automorphisms between some of the vertices. 
Thus, exactly solving the graph matching problem in general is not feasible if $p>0$. However, as the fraction of shuffled vertices gets larger the value of $\hat{p}_k(A)$ increases for some of the networks, which suggests that partial matching is still possible. 
The political blogs and the C. elegans networks have the highest values for the measured quantities, $\hat{X}_k(A)$ and $\hat p_k(A)$, which can be explained by  their large average degree and  small clustering coefficient.
On the other hand, the protein-protein interaction and arXiv citation networks have the lowest values, possibly because of their low density. 
The probability bounds in the karate network might be very conservative since the results of Theorem \ref{thm:asym} are asymptotic, and $n$ is small for this network.

\begin{figure}
	\centering
	\includegraphics[width=0.7\textwidth]{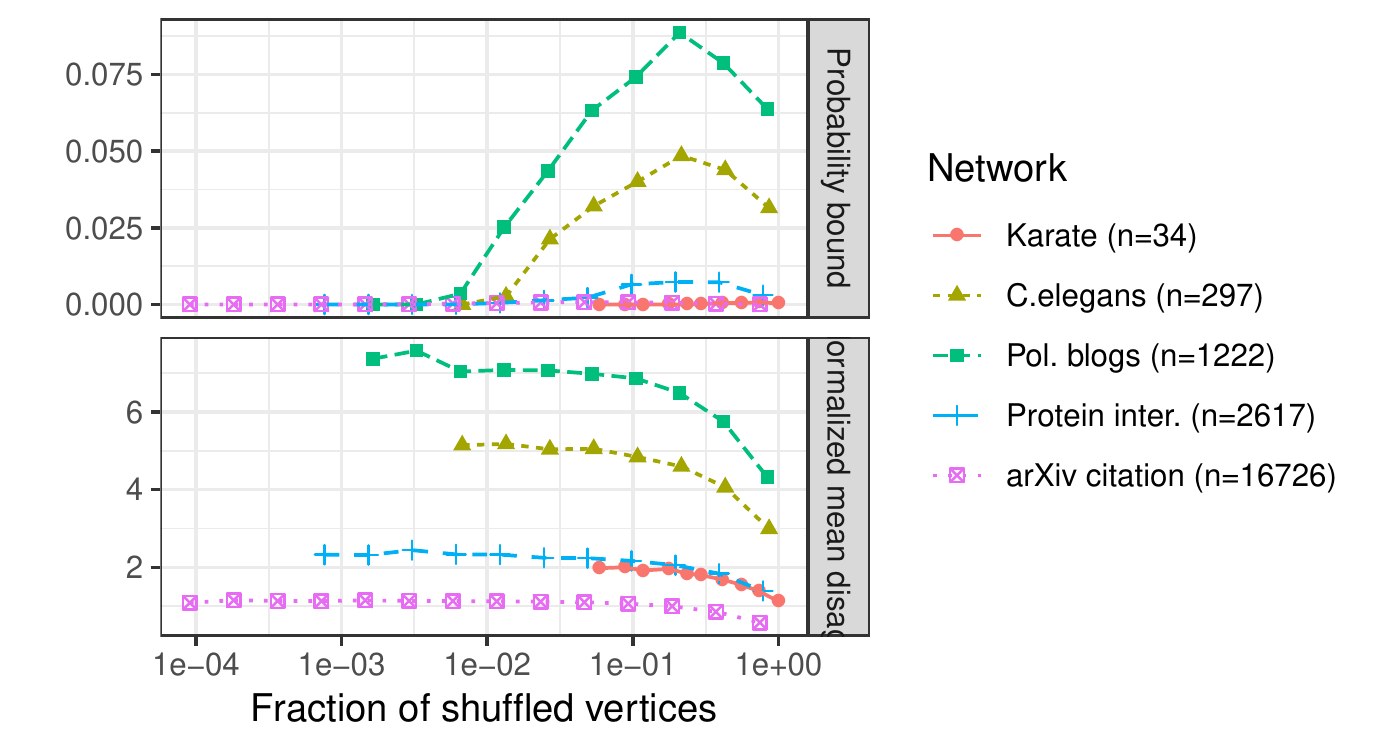}
	\caption{Measures of matching feasibility for different real-world networks as a function of the fraction of shuffled vertices $k/n$. These measures increase with a large number of shuffled vertices, suggesting that partial matching is possible.}
	\label{fig:exp-real-data}
\end{figure}

In practice, computing the MLE of the unshuffling permutation is computationally unfeasible, as solving the GMP requires to optimize a loss function over $\Pi_n$. 
Thus, 

To validate the results above, we measure the accuracy of a graph matching solution of the problem 
\eqref{eq:GMP} given graphs $A$ and $B$, where $B$ is generated using a uniformly corrupting channel $(p,I)$. Then, we perform graph matching using the FAQ programming algorithm ; we use the true parameter $I$ as the initialization value of the FAQ algorithm in order to check whether $I$ (the unshuffling permutation) is a local minimum of the matching objective function. 
While this is not finding a globally optimal solution to the graph matching problem in general, this strategy provides a useful surrogate for the difficulty/feasibility of the deanonymization task.

For each network $A$, we generate 30 independent random graphs $B$ from the uniformly corrupting channel with the same value of $p$,  and measure the average matching accuracy of the solution $\hat{P}^{\text{FAQ}}$. 
This process is repeated
for different values of $p\in\{10^{-3}, 10^{-2.5}, \ldots, 10^{-0.5}\}$. 
The accuracy of the solution is measured in two ways: 
First, we calculate the total matching accuracy as the percentage of vertices that are correctly matched; i.e., we compute $\frac{1}{n}\sum_{i=1}^n\textbf{1}(\hat{P}^{\text{FAQ}}_{ii}=1)$. 
As mentioned in Section \ref{sec:quasi}, matching  periphery vertices is usually hard in practice, thus, we also investigate whether it is possible to correctly match the core vertices by measuring the accuracy of matching the vertices with the highest degrees. 
If $j_1,\ldots, j_n$ is an ordering of the vertex indexes according to their degree, so that $\sum_{i=1}^nA_{j_ui} \geq \sum_{i=1}^nA_{j_vi}$ if $u<v$, then the accuracy of matching a fraction $c$ of the vertices with the highest degree is given by $\frac{1}{\lfloor cn\rfloor}\sum_{i=1}^{\lfloor cn\rfloor}\textbf{1}(\hat{P}^{\text{FAQ}}_{j_ij_i}=1)$, with $c\in(0,1]$. 

%\begin{figure}
	%\centering

	%\caption{Matching accuracy calculated as the fraction of all the vertices that are correctly matched for a given network and a random graph generated from the uniformly corrupting channel model with parameters $(p,I)$. The matching is performed with the FAQ algorithm initialized at $I$.}
	%\label{fig:realdata-matchingaccuracy}
%\end{figure}

%\begin{figure}
	%\centering
	%\includegraphics[width=0.6\textwidth,height=2.8in]{Highdegree.pdf}
	%\caption{Matching accuracy calculated as the fraction of high degree vertices that are correctly matched. Vertices are ranked from highest to lowest degree, and the accuracy is calculated considering only the vertices with a normalized degree rank smaller than $c\in(0,1]$. The plots show the resulting accuracy as a function of the fraction $c$ of vertices considered, and for different values of the corrupting channel probability $p$.}
	%\label{fig:realdata-highestdegree}
%\end{figure}

\begin{figure}
	\centering
	\begin{subfigure}[b]{0.48\textwidth}
		\includegraphics[width=\textwidth]{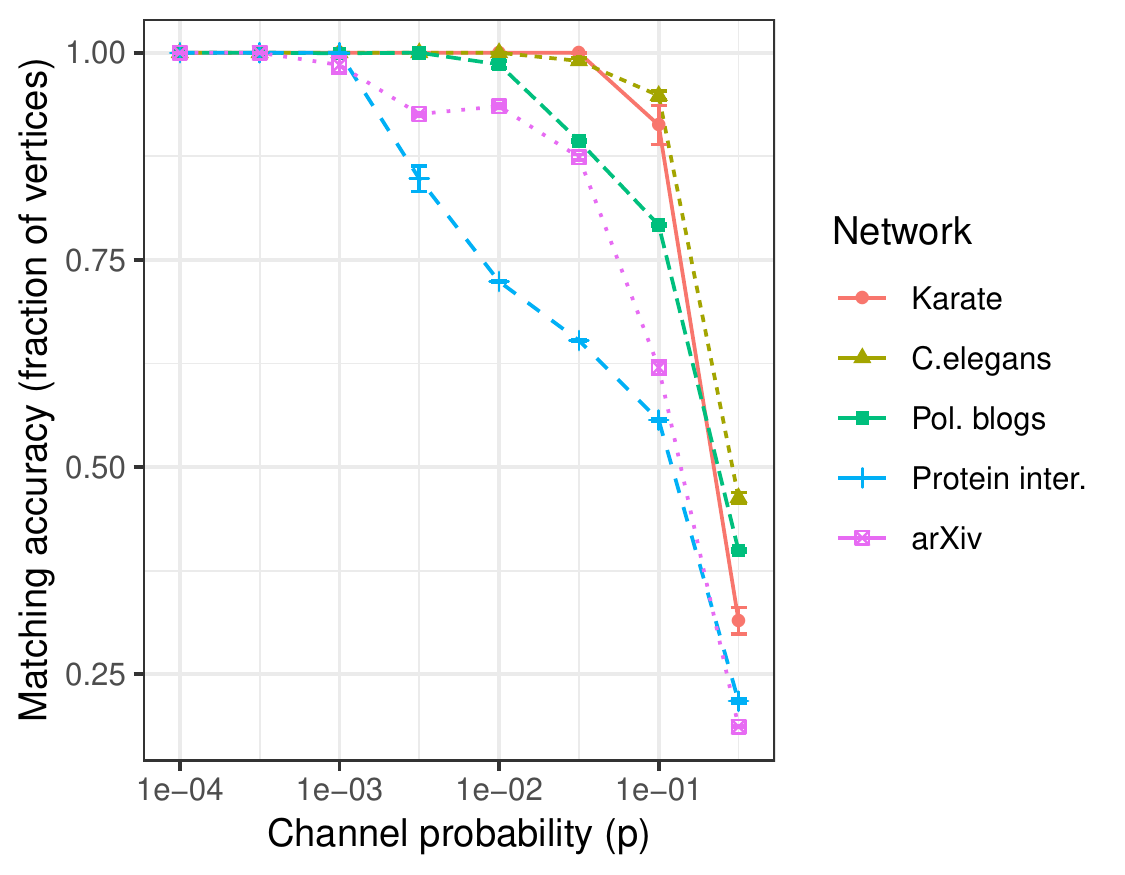}
		\caption{Overall matching accuracy}
		\label{fig:realdata-matchingaccuracy}
	\end{subfigure}
	\begin{subfigure}[b]{0.48\textwidth}
		\includegraphics[width=\textwidth,height=2.8in]{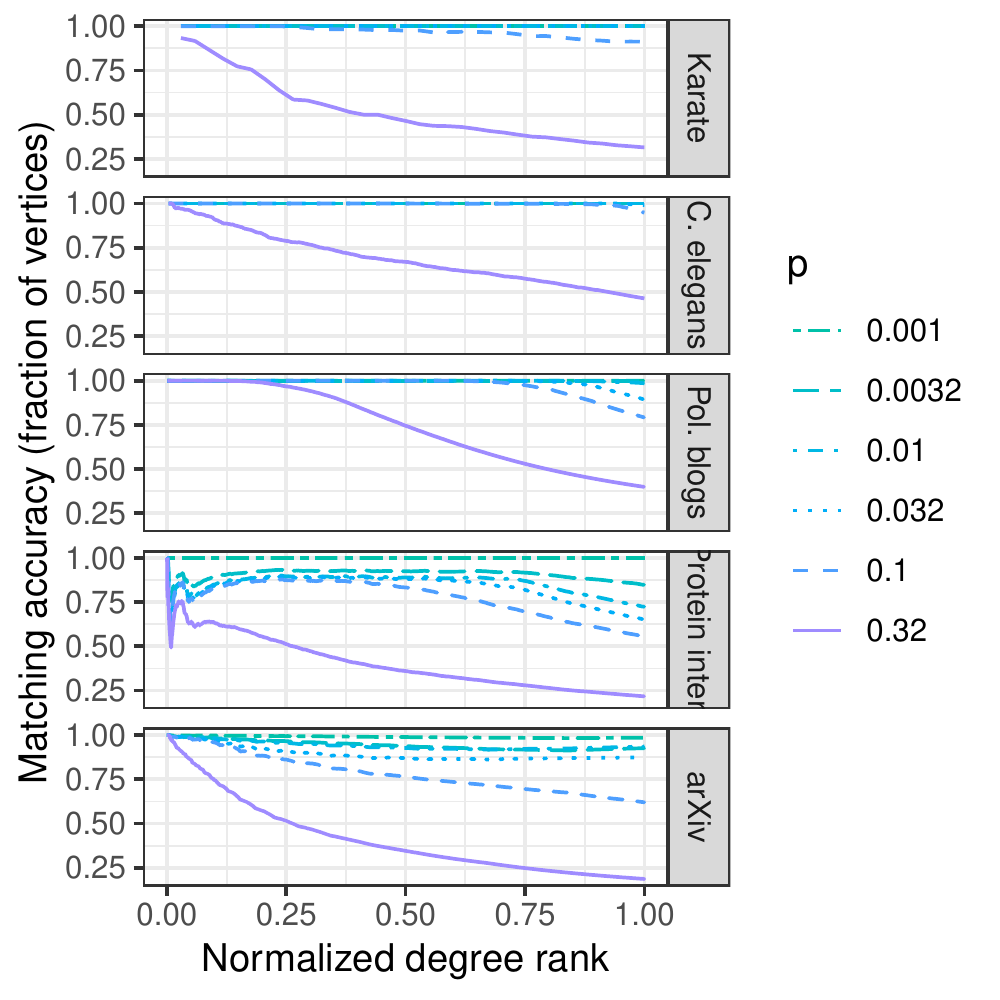}
		\caption{Matching accuracy by degree}
		\label{fig:realdata-highestdegree}
	\end{subfigure}
	\caption{Matching accuracy calculated as the fraction of vertices that are correctly matched. The left plot shows the accuracy of matching correctly all the vertices as a function of the channel probability. In the right plot, vertices are ranked from highest to lowest degree, and the accuracy is calculated considering only the vertices with a normalized degree rank smaller than $c\in(0,1]$. The plots show the resulting accuracy as a function of the fraction $c$ of vertices considered, and for different values of the corrupting channel probability $p$.}
\end{figure}

The overall matching accuracy and the matching accuracy of the vertices with highest degrees are shown in Figures \ref{fig:realdata-matchingaccuracy} and \ref{fig:realdata-highestdegree}. 
When looking to all the vertices aggregated in Figure \ref{fig:realdata-matchingaccuracy}, the accuracy in the arXiv and protein interaction networks decreases fast as $p$ increases, which agrees with the predicted matching feasibility of Figure \ref{fig:exp-real-data}.  The  Karate and C. elegans networks have the highest matching accuracy for most of the values of $p$. 
This suggests that the bound obtained in Theorem \ref{thm:asym} might be conservative for networks with a small number of vertices. However, when looking only to the vertices with the highest degree (Figure \ref{fig:realdata-highestdegree}), the matching accuracy is higher in the political blogs networks, in which almost 25\% of the vertices with the highest degree can be matched accurately even for large values of $p$; this is as expected by the results on Figure \ref{fig:exp-real-data} in which this network is the more resistant to noise for large values of $k$. 
Figure \ref{fig:realdata-highestdegree} also shows the difficulty of matching vertices in the protein-protein interaction and arXiv citation networks, in which even the vertices with the highest degrees are usually incorrectly matched. 
This is especially true for the protein-protein interaction network, which has a core-periphery structure in which high-degree nodes are highly connected between each other. As observed in Table \ref{table:summary-stats}, this graph is characterized by a large clustering coefficient and a heavy tailed degree distribution.

\section{Discussion}
\label{sec:discussion}
The inexact graph matching problem aims to find the alignment that minimizes the number of edge disagreements between a pair of graphs. 
In this paper, we have shown that this intuitive notion of matching coincides with a maximum likelihood estimator in errorfully observed graphs, which is formally defined using the corrupting channel model. 
This model is able to accommodate different correlation structures between the edges of a pair of graphs, and many other popular models for the paired and correlated networks are encompassed by this framework.  
Within this model, we derive necessary and sufficient conditions to determine whether the MLE is consistent, and introduce a relaxed notion of consistency in which all but a small fraction of the vertices are correctly matched. 

Since the distribution of the corrupting channel model conditions on a given graph $A$, the consistency results we presented here only depend on the structure of $A$ and the channel noise. This property allowed us to derive conditions that can be used to  check whether  the MLE is consistent for a given family of graphs, and hence whether it is feasible to solve the GMP. We used these results  to study matching feasibility  in some popular random graph models, unifying some previous matchability results within our framework and introducing some new ones as well. 
Our results were tested in simulated and real networks, and we introduced a statistic that can be used in practice to estimate matching feasibility.  
In addition, we believe that our results can be used to study the feasibility of solving the GMP in other graph models of interest.

This paper studies the information limits of the graph matching problem in the corrupting channel model, and currently there is no known efficient algorithm for finding the solution to the graph matching problem in this model framework. 
Hence, finding a polynomial-time algorithm to solve the GMP in this model and studying the corresponding computational limits are important open questions for future work. 
These questions have been partially addressed  in a number of settings (see for example \cite{JMLR:v15:lyzinski14a, shirani2017seeded, zhang2018unseeded,zhang2018consistent}),  but existing methods usually require seeds to initialize the method, or impose restrictive constraints in the type of graphs or edge correlations that can be handled. Our framework offers a new insight which can be potentially useful in this direction.

\vspace{3mm}
\noindent {\bf Acknowledgements:}
This material is based on research sponsored by the Air Force Research Laboratory and DARPA, under agreement number FA8750-18-2-0035. The U.S. Government is authorized to reproduce and distribute reprints for Governmental purposes notwithstanding any copyright notation thereon. The views and conclusions contained herein are those of the authors and should not be interpreted as necessarily representing the official policies or endorsements, either expressed or implied, of the Air Force Research Laboratory and DARPA, or the U.S. Government.

%%%%%%%%%%%%%%%%%%%%%%%%%%%%%%%%%%%%%%%%%%%%%%%%%%%%%%%%%%%%%%%%%%%%%%%%%%%%%%%%%%%%%%%%%%%%
\appendix
\section{Proof of main results}

%Herein we collect the proofs of the main results in the paper
\subsection{Proof of Theorems \ref{thm:asym} and \ref{thm:asym-het}}
\label{sec:pfasym}
Note that we will prove the more general Theorem \ref{thm:asym-het} here, with the proof of Theorem \ref{thm:asym} then following immediately as a special case.
\begin{proof}[Proof of Theorem
\ref{thm:asym-het} part i.]
	We will first prove part i.\@ of the theorem. We show that under the conditions of the theorem, with high probability we have
	\begin{equation}
	\|A_n-P_n^TB_nP_n\|_F^2 \leq \|A_n-Q_n^TP_n^TB_nP_nQ_n\|_F^2, \label{eq:proof-edgediff}    
	\end{equation}
		for any $Q_n\neq I_n$, where $I_n$ is the identity matrix. The proof proceeds by expressing the difference between the left and the right hand side of Equation~\eqref{eq:proof-edgediff} can be expressed as a sum of independent random variables, and the tail of its distribution can be sufficiently bounded to show that Equation~\eqref{eq:proof-edgediff} holds.
		
	Observe that $P_n^TB_nP_n\sim C(A_n, \Psi_n^{(1)}, \Psi_n^{(2)}, I_n)$, so to this end, for a graph $A\in\mathcal{G}_n$, define $\widetilde{B}\sim C(A,\Psi_n^{(1)}, \Psi_n^{(2)},I_n)$ to be the channel corrupted $A$ with parameters $(\Psi_n^{(1)}, \Psi_n^{(2)},I_n)$.
	For each $Q\in\Pi_n$, define
	\begin{align*}
	X_Q(A)&=X_Q:=\frac{1}{2}\|A-Q^TAQ\|^2_F\\
	\widehat{X}_Q(A)&=\widehat{X}_Q:=\frac{1}{2}(\|A-Q^T\widetilde{B}Q\|^2_F-\|A-\widetilde{B}\|^2_F).
	\end{align*}
	Consider a fixed $Q^T\in\Pi_{n,k}$, and let the permutation associated with $Q^T$ be denoted via $\sigma$.
	Note that for any matrix $M\in\mathbb{R}^{n\times n}$,
	$[Q^TMQ]_{ij}=M_{\sigma(i),\sigma(j)}$. Recall that $V$ is the set of all vertices, which is isomorphic to $[n]$, and denote by $\binom{V}{2}=\{(i,j)\in[n]\times [n], i<j\}$ to the set of vertex pairs. Letting 
	\begin{align*}
	\Delta_1&:=\left|\left\{\,\,\{i,j\}\in\binom{V}{2}\text{ s.t. }A_{i,j}\neq A_{\sigma(i),\sigma(j)},\,A_{i,j}= \widetilde{B}_{\sigma(i),\sigma(j)}   \right\}\right|\\
	\Delta_2&:=\left|\left\{\,\,\{i,j\}\in\binom{V}{2}\text{ s.t. }A_{i,j}= A_{\sigma(i),\sigma(j)},\,A_{i,j}\neq \widetilde{B}_{\sigma(i),\sigma(j)}   \right\}\right|,
	\end{align*}
	we have that 
	\begin{align*}
	\frac{1}{2}\|A-Q^T\widetilde{B}Q\|^2_F&=X_Q+\Delta_2-\Delta_1,\\ \frac{1}{2}\|A-\widetilde{B}\|^2_F&=\frac{1}{2}\|Q^TAQ-Q^T\widetilde{B}Q\|^2_F=\Delta_2+\Delta_1.
	\end{align*}
	And hence
	$\widehat{X}_Q=X_Q-2\Delta_1$. Let $\tilde\Delta_1\sim\text{Binom}(X_Q,\tilde{p}_n)$ be an independent binomial random variable.
	Noting that $\Delta_1\sim\sum_{i=1}^{X_Q}	\text{Ber}(q_i)$ is a sum of independent Bernoulli random variables, with parameters $q_i \leq \tilde{p}_n$ (each $q_i$ is equal to its corresponding corrupting probability $\Psi^{(k)}_{i'j'}$, which is upper bounded by $\tilde{p}$), we have that $\p(\Delta_1\geq \tau) \leq \p(\tilde\Delta_1\geq\tau)$ for all $\tau\geq 0$, and hence
	\begin{align}
	\label{eq:XQk}
	\p(\widehat{X}_Q\leq 0)&=\p( \Delta_1\geq X_Q/2) \leq \text{exp}\left\{-X_Q H(1/2,\tilde{p}_n)  \right\}\notag\\
	&=\text{exp}\left\{\frac{1}{2}X_Q \log(4\tilde p_n(1-\tilde p_n))\right\},
	\end{align}
	where $H(1/2,p)$ is the relative binomial entropy \cite[Theorem 1]{arratia1989tutorial}. 
	If $n\geq n_1$, then the Assumption in Eq. \ref{eq:growthrate} implies  
	\begin{align}
	\label{eq:XQkk}
	    \p(\widehat{X}_Q\leq 0)\leq \text{exp}\left\{-3k\log(n) \right\},
	\end{align}
	and this bound holds uniformly for all $Q^T\in\Pi_{n,k}$.
	
	Consider now $n\geq n_1$.
	Let $B_n\sim C(A_n,\Psi^{(1)}_n, \Psi^{(2)}_n,P_n)$ and $\widetilde{B}_n\sim C(A_n,\Psi^{(1)}_n, \Psi^{(2)}_n,I_n)$.
For each $Q\in\Pi_n$, let 
	\begin{align*}
	\widehat{Y}_Q(A)&=\widehat{Y}_Q:=\frac{1}{2}(\|A_n-Q^TB_nQ\|^2_F-\|A_n-P_n^TB_nP_n\|^2_F)\\
	\widehat{X}_Q(A)&=\widehat{X}_Q:=\frac{1}{2}(\|A_n-Q^T\widetilde{B}_nQ\|^2_F-\|A_n-\widetilde{B}_n\|^2_F).
	\end{align*}
	Note that $\widehat{X}_{P_n^TQ}=\widehat{Y}_{Q}$, and hence,
	\begin{align*}
	\p(\exists \,Q\neq P_n\text{ s.t. }\widehat{Y}_Q\leq 0)&=
	\p(\exists \,P_n^T Q\neq I\text{ s.t. }\widehat{X}_{P_n^T Q}\leq 0)\\
	& =\p(\exists \,\tilde Q\neq I\text{ s.t. }\widehat{X}_{\tilde Q}	\leq 0)
	\end{align*}
	For each $k>1$, note that $|\Pi_{n,k}|\leq n^k$.
	For each $k$, applying Eq. \eqref{eq:XQkk} uniformly for each $\tilde{Q}\in\Pi_{n,k}$ and summing over $k$ yields
	\begin{align}
	\p(\exists \,\tilde Q\neq I\text{ s.t. }\widehat{X}_{\tilde Q}	\leq 0)\leq\sum_{k= 2}^n \text{exp}\left\{-2k\log n  \right\} \leq\text{exp}\left\{-4\log n+\log n  \right\}=n^{-3},\label{eq:proof-asymp-1-to-2a}
	\end{align}
and therefore
\begin{align}
\p(\exists \,Q\neq P_n\text{ s.t. }\widehat{Y}_Q\leq 0)&=
\p\left( \argmin_{P\in\Pi_n}\|A_n-P^TB_nP\|_F^2\neq\{P_n\}\right) \label{eq:proof-asymp-1-to-2b}\\
&=\p(\exists\, \text{ an MLE, }\hpi,\text{ of }P_n\text{ s.t. }\hpi\neq P_n)\leq n^{-3}.\nonumber
\end{align}
	Therefore, with probability $1$ it holds that for all but finitely many $n$ (by the Borel-Cantelli lemma as $n^{-3}$ is finitely summable) the unique MLE of $P_n$ is $\hpi=P_n$, and any sequence of MLE's, $(\hpi)_{n=n_0}^{\infty}$, is strongly consistent.
\vspace{ 2mm}

\noindent\emph{Proof of Theorem \ref{thm:asym-het} part ii.} 
Note that condition Eq. \ref{eq:gr22} implies
\begin{align}
    		%\min_{Q_n\in\Pi_{n,k_n}} \|A_n Q_n-Q_n A_n\|^2_F&=\omega(1)\notag\\
		\label{eq:gr2}
		&\max_{Q_n\in S_{n}}\log\left(\frac{1}{4\breve{p}_n(1-\breve{p}_n)}\right)\left(\|A_n -Q_n^T A_nQ_n\|^2_F\right)=o\left(\log(n) \right)
\end{align}
%In this part, we assume that $\min_{Q_n\in S_n}\|A-Q_n^TAQ_n\|_F^2=\omega(1)$ (the case in which this condition does not hold is covered by part iii. of Theorem~\ref{thm:asym}).
As in the proof of part i., we will work with $P_n=I_n$, as the consistency results for the MLE estimating $I_n$ transfer immediately to consistency results for the MLE estimating a more general $P_n$\
Consider the notation as in the proof of part i, and consider the $\Theta(n)$ disjoint $Q_n$ permutations with ${Q_n}\in\cup_{i=2}^{k_n}\Pi_{n,i}\setminus{I_n}$
(where $S_{n}$ be the set of these $Q_n$) satisfying Eq.  (\ref{eq:gr22}), (\ref{eq:gr23}) as posited in Theorem~\ref{thm:asym-het}.
Let $Z_n=\sum_{Q_n\in S_{n,k}} \mathds{1}\{\widehat{X}_Q\leq 0\}$.

Let $Q\in S_{n}$ be any of the permutations as defined above. We have that
$\Delta_1\sim\sum_{i=1}^{X_Q}\text{Ber}(q_i)$ for some independent Bernoulli random variables with parameters $\breve{p}_n\leq q_i\leq \tilde{p}_n$, and let
$\breve\Delta_1\sim\text{Binom}(X_Q,\breve{p}_n)$.
We will make use of the following Theorem from \cite{zhang2018non} in our proof.
\begin{theorem}[Theorem 9 from \cite{zhang2018non}]
Suppose that $X$ is centralized binomial distributed
with parameters $(k,p)$. For any $\beta > 1$, there exist constants $c_\beta$, $C_\beta > 0$ that only rely on $\beta$, such
that
\begin{align}
\label{eq:az1}
    \p(X \geq x)\begin{cases}
    \geq c_\beta\text{exp}(-C_\beta k h_p(p+\frac{x}{k})),&\text{ if }0\leq x\leq \frac{k(1-p)}{\beta}\text{ and }x+kp\geq 1;\\
    =1-(1-p)^k,&\text{ if }0\leq kp+x<1
    \end{cases}\\
    \label{eq:az2}
        \p(X \leq -x)\begin{cases}
    \geq c_\beta\text{exp}(-C_\beta k h_p(p-\frac{x}{k})),&\text{ if }0\leq x\leq \frac{kp}{\beta}\text{ and }x+k(1-p)\geq 1;\\
    =1-p^k,&\text{ if }0\leq k(1-p)+x<1
    \end{cases}
\end{align}
where $h_u(v)=v\log(v/u)+(1-v)\log(\frac{1-v}{1-u})$.
\end{theorem}
\noindent In the notation of the above theorem, let $X=\breve\Delta_1-X_Q\breve p_n$ (so that $k=X_Q$ and $p=\breve p_n$) and $x=X_Q(1/2-p_n)$.
We have that
$$x=X_Q(1/2-\breve p_n)\leq X_Q(1-\breve p_n)/2,
$$
and we can take $\beta=2$.
Applying Eq. \eqref{eq:az1} and the fact that  $\p(\breve{\Delta}_1\geq \tau)\leq \p(\Delta_1\geq \tau)$ for all $\tau\geq 0$, then implies (where $c$ and $C$ are constants independent of $n$)
\begin{align}
p_Q:=\p(\Delta_1\geq X_Q/2 )&\geq \p(\breve\Delta_1-X_Q\breve p_n\geq X_Q(1/2-\breve p_n))\nonumber\\   
&\geq c \text{ exp}\left\{-C X_Q \frac{1}{2}\log\left(\frac{1}{4\breve{p}_n(1-\breve p_n)} \right)\right\}\nonumber\\
&=\Omega\left( c \text{ exp}\left( -\frac{1}{2}\log(n) \right)\right) \label{eq:proof-lowboundpq}
\end{align}
where the last line follows from Eq. \eqref{eq:gr2}.

Consider now disjoint $Q=Q_n$ and $Q'=Q_n'$ in $S_{n,k_n}$. %recall that disjoint means that
%$$\{i \text{ s.t. }\sigma_{Q}(i)\neq i\}\cap \{i \text{ s.t. }\sigma_{Q'}(i)\neq i\}=\emptyset.$$
Each $\widehat X_Q=\sum_{i,j}A_{i,j}(B_{i,j}-B_{\sigma_Q(i),\sigma_Q(j)})$ is a function of $\Theta(nk_n)$ independent Bernoulli random variables (the $B_{i,j}$ that are permuted).
The disjointedness of 
$Q$ and $Q'$ implies that there are at most $2k^2$ common $B_{i,j}$ that appear in both $\widehat X_Q$ and $\widehat X_{Q'}$,
and we can decompose 
$\widehat X_{Q'}$ into 
$$\widehat X_{Q'}=Y+W,$$ where $Y$ is independent of $\widehat X_Q$ and $|W|\leq 2k_n^2$.
Note that $k_n=\Theta(1)$ by construction.
We then have that (where $\Delta_1'$ is the $\Delta_1$ associated with $\widehat X_{Q'}$)
\begin{align}
\operatorname{Cov}\left( \mathds{1}\{\widehat X_Q\leq 0\}, \mathds{1}\{\widehat X_{Q'}\leq 0\}\right) & =\p(\widehat X_Q\leq 0,\widehat X_{Q'}\leq 0)-\p(\widehat X_Q\leq 0)\p(\widehat X_{Q'}\leq 0) \nonumber\\
&\leq \p(Y\leq 2k_n^2,\widehat X_{Q'}\leq 0)-\p(Y\leq -2k^2)\p(\widehat X_{Q}\leq 0)\nonumber\\
&=\p(\widehat X_{Q}\leq 0)\p(-2k_n^2\leq Y\leq 2k_n^2)\nonumber\\
&\leq \p(\widehat X_{Q}\leq 0)\p(-4k_n^2\leq \widehat X_{Q'}\leq 4k_n^2)\nonumber\\
&=p_Q \p(X_{Q'}/2-2k_n^2\leq \Delta_1' \leq X_{Q'}/2+2k_n^2)\label{eq:proof-covbound}
%&\leq C p_Q \frac{1}{\sqrt{X_{Q'}}} 
\end{align}
Observe that conditions~\eqref{eq:gr22-het} and \eqref{eq:gr23-het} imply that  that $X_Q=\omega(1)$ and that $(1/2 -\breve{p}_n)^2=o(1)$. The variable $\Delta_1'=\sum_{i=1}^{X_{Q'}}\text{Ber}(q_i)$ is a sum of $X_{Q'}$ independent binary variables with $\breve{p}_n\leq q_i\leq \tilde{p}_n$, and therefore, by the Lindeberg-Feller's Central Limit Theorem, $s_{n}^{-1}(\Delta_1'-\mathbb{E}[\Delta_1'])$ converges in distribution to a standard normal, with
$$X_{Q'}\breve{p}_n\leq \mathbb{E}[\Delta_1']\leq X_{Q'}\tilde{p}_n,$$
$$X_{Q'}\breve{p}_n(1-\breve{p}_n) \leq s_n^2 \leq X_{Q'}\tilde{p}_n(1-\tilde{p}_n).$$
Therefore, setting $H_n:=\sqrt{X_{Q'}}\left(\frac{1/2 - \mathbb{E}[\Delta_1']/X_{Q'}}{s_n/\sqrt{X_{Q'}}} \right)$, we have that for $n$ sufficiently large
\begin{align*}
    \p\left(\Delta_1'-\frac{X_{Q'}}{2} \in\left[-2k_n^2,2k_n^2\right] \right) & = \p\left( \frac{\Delta_1' - \mathbb{E}[\Delta_1']}{s_n}\in\left[H_n -\frac{2k_n^2}{s_n}, H_n  +\frac{2k_n^2}{s_n}\right]\right)\\
    &\leq C_1\frac{k_n^2}{s_n}\leq 
    \frac{C_2}{\sqrt{X_{Q'}}}.
\end{align*}
for some constants $C_1, C_2$. Combining this bound with Equation~\eqref{eq:proof-covbound}, we have
$$\operatorname{Cov}\left( \mathds{1}\{\widehat X_Q\leq 0\}, \mathds{1}\{\widehat X_{Q'}\leq 0\}\right)\leq C_2 p_Q \frac{1}{\sqrt{X_{Q'}}}  $$ 
%(note the last equality follows from the Local Central Limit Theorem applied to $\Delta_1'$; see
%\cite{lesigne2005heads} Theorem 9.1) %\texttt{---Need to make an argument for $\breve{\Delta}_1$ and $\tilde{\Delta}_1$ above instead of $\Delta_1$ (see thm 2.2 of this \url{http://jirss.irstat.ir/article-1-127-en.pdf})----}.
% of The
% We next turn our attention to part ii. of the theorem.
% 	Note that if $p< 1/2$ then
% 	\begin{align}
% 	\label{eq:ldb}
% 	p_Q:&=\p(\widehat{X}_Q\leq 0)=\p(\Delta_1\geq X_Q/2)\notag=\Theta\left(e^{-X_Q H(1/2,p_n)}\left((1/2-p_n) \sqrt{X_Q}\right)^{-1}\right)\notag\\
% 	&=\Theta\left(e^{\frac{1}{2}X_Q \log(4p_n(1-p_n))}\left((1/2-p_n) \sqrt{X_Q}\right)^{-1}\right)
% 	\end{align}
% 	where the second equality follows from the large deviations bound on the binomial distribution  \cite[Theorem 2]{arratia1989tutorial}. Consider $\Theta(n)$ disjoint ${Q_n}\in\Pi_{k_n,n}$ (let $S_{n,k_n}$ be the set of these $Q_n$),
% 	and let $Z_n=\sum_{Q_n\in S_{n,k}} \mathds{1}\{\widehat{X}_Q\leq 0\}$. 
	We now have that 
	$$\text{Var}(Z_n)\leq \e(Z_n)+\sum_{Q_n\in S_{n,k}}\sum_{Q'_n\neq Q_n}p_{Q_n}\frac{C}{\sqrt{X_{Q_n'}}}.$$
	The second moment method (see \cite{alon} Theorem 4.3.1) can be applied to show 
	\begin{align*}
	\p(\exists\,Q_n\in\Pi_{n,k_n}\text{ s.t. }\widehat{X}_{Q_n}\leq 0)&=\p(Z_n>0) \geq 1-\frac{\text{Var}(Z_n)}{(\e Z_n)^2 }\\
	&\geq1-\frac{\e Z_n }{(\e Z_n)^2}-\frac{\sum_{Q_n\in S_{n,k}}\sum_{Q'_n\neq Q_n}p_{Q_n}\frac{C}{\sqrt{X_{Q_n'}}}}{\sum_{Q_n}\sum_{Q'_n}p_{Q_n}p_{Q_n'}}.
	\end{align*}
By the assumption in Eq. \eqref{eq:proof-lowboundpq}, we have that
	\begin{align*}
	\e(Z_n)=\Omega(ne^{-\frac{1}{2}\log n})=\omega(1),
	\end{align*}
By the assumption in Eq.\eqref{eq:gr23-het}, we have that $\breve{p}_n\rightarrow\frac{1}{2}$ (otherwise, the left hand side of Eq. \eqref{eq:gr23-het} is at least of a constant order). Combining this fact with $\breve{p}_n\leq \tilde{p}_n$ and Eq.
\eqref{eq:gr22-het}, we have that $\min_{Q_n\in S_n}\|A_n-Q_n^TA_nQ_n\|_F^2=\omega(1)$.
%that $(\frac{1}{2}-\breve{p}_n)^2=O\left(\frac{1}{n}\right)$, and observing that $(\frac{1}{2}-\breve{p}_n)^2\sim \log\left(\frac{1}{4\breve{p}_n(1-\breve{p}_n)}\right)$ and $X_Q=O(n)$ (because $Q$ permutes at most $k_n=\Theta(1)$ rows of $A_n$), 
Combining these facts, we have that for each $Q\in S_{n}$,
\begin{align*}
-X_Q \log\left(\frac{1}{4\breve p_n(1-\breve p_n)}\right)+\log(X_Q)& = \log(X_Q)\left(1 - \frac{X_Q}{\log X_Q} \log\left(\frac{1}{4\breve p_n(1-\breve p_n)}\right)\right) \\
& \geq \log(X_Q)\left(1 - \max_{Q\in S_n}\frac{X_Q}{\log X_Q} \log\left(\frac{1}{4\breve p_n(1-\breve p_n)}\right)\right)\\
& \geq \min_{Q\in S_n}\log(X_Q)\left(1 - o(1)\right)\\
& =\omega(1).   
\end{align*}
Let 
$$m_n:=\min_{Q\in S_{n}}-X_Q \log\left(\frac{1}{4\breve p_n(1-\breve p_n)}\right)+\log(X_Q),$$
and note that $m_n=\omega(1)$.
Therefore,
	\begin{align*}
	\frac{1/\sqrt{X_{Q'}}}{p_{Q'}}&\leq
	c\frac{1}{\sqrt{X_{Q'}}} \text{ exp}\left\{C X_{Q'} \frac{1}{2}\log\left(\frac{1}{4\breve{p}_n(1-\breve{p}_n)} \right)\right\}\\
	&=c\,\text{exp}\left\{C X_{Q'} \frac{1}{2}\log\left(\frac{1}{4\breve{p}_n(1-\breve{p}_n)}\right)-\frac{1}{2}\log(X_{Q'})
\right\}\\
&\leq c\,\text{exp}\{ -C m_n\}=o(1).
%&=\Theta\left((	1/2-p_n) e^{-\frac{1}{2}X_{Q_n} \log(4p_n(1-p_n))}\right)\\
% 	&=\Theta\left((	1/2-p_n)\,\, e^{X_{Q_n} (	1/2-p_n)^2}\right)
% 	=o((1/2-p_n) n^{\delta/2})=o(n^{-\delta/2})
	\end{align*}
	We then have that
	\begin{align*}
	\frac{\sum_{Q_n\in S_{n,k}}\sum_{Q'_n\neq Q_n}p_{Q_n}\frac{1}{\sqrt{X_{Q_n'}}}}{\sum_{Q_n}\sum_{Q'_n}p_{Q_n}p_{Q_n'}}&\leq 
		c\,\text{exp}\{ -C m_n\}	\frac{\sum_{Q_n\in S_{n,k}}\sum_{Q'_n\neq Q_n}p_{Q_n}p_{Q_n'}}{\sum_{Q_n}\sum_{Q'_n}p_{Q_n}p_{Q_n'}}
	&=o(1).
	\end{align*}
	$$\p(\exists\,Q_n\in S_n\text{ s.t. }\widehat{X}_{Q_n}\leq 0)=\p(Z_n>0)
	\geq 1-o(1).$$
	Therefore, for any $1>\epsilon>0$
	\begin{align*}
	    \p(\max_{P\in \mathcal{P}^*}\|P-I_n\|_F^2 \leq \epsilon)\rightarrow 0
	\end{align*}
	giving the desired inconsistency result.
	
\noindent \emph{Proof of Theorem \ref{thm:asym-het} part iii.}
As in the proofs of part i. and ii., we will work with $P_n=I_n$, as the consistency results for the MLE estimating $I_n$ transfer immediately to consistency results for the MLE estimating a more general $P_n$.
Consider the notation as in the proof of parts i. and ii.
To prove part iii., consider first the case that $\breve p_n$ is bounded away from zero. 
Note that in this case, Equation~\eqref{eq:thm-condition3-2} (combined with the bound in Eq. \eqref{eq:az1}, and the fact that $\breve{\Delta}_1$ is stochastically less or equal than $\Delta_1$) implies that 
	%Note that by equation  \eqref{eq:ldb} we have that for each $Q_n$, 
	$$p_{Q_n}
\begin{cases}
    	\geq c \text{ exp}\left\{-C X_{Q_n} \frac{1}{2}\log\left(\frac{1}{4\breve{p}_n(1-\breve{p}_n)} \right)\right\}=\Theta(1)&\text{ if }X_{Q_n}/2\geq 1\\
    	\geq 1-(1-\breve{p}_n)^{X_{Q_n}} =\Theta(1)&\text{ if }X_{Q_n}/2< 1\\
    	\end{cases}$$ 
	On the other hand, if $\breve p_n\rightarrow 0$ as $n$ goes to infinity, then  Equation~\eqref{eq:thm-condition3-2} implies that $X_{Q_n} = o(1)$, and hence there exists infinitely many $n$'s for which $X_{Q_n}=0$ and 
	the MLE would not be uniquely defined.
	Therefore, for any $1>\epsilon>0$
	\begin{align*}
	    \limsup \p(\max_{P\in \mathcal{P}^*}\|P-I_n\|_F^2 \leq \epsilon)<1
	\end{align*}
% 	so $\widehat{X}_{Q_n}\leq 0$, which means that $p_{Q_n}=1$. 
% 	Therefore, there exists infinitely many values of $n$ such that
% 	\begin{align*}
% 	\p(\exists\,Q_n\in\Pi_{n}\text{ s.t. }\widehat{X}_{Q_n}\leq 0)&\geq p_{Q_n} =\Theta(1),
% 	\end{align*}
	and so the MLE is not weakly consistent.
\end{proof}

\begin{proof}[Proof of Theorem~\ref{thm:asym}] Observe that under the settings of the theorem, using the notation of Theorem~\ref{thm:asym-het} we have that
$$\tilde{p}_n=p_n = \breve{p}_n.$$
From this observation, the result follows.
\end{proof}

\subsection{Proof of Corollary \ref{cor:examples}}
\label{sec:pfexamples}
\begin{proof}%[Proof of Corollary \ref{cor:examples}]
	By Lemma 4 in \cite{rel}, we have that for $n$ sufficiently large and all $Q\in\Pi_{n,k}$,
	\begin{equation}
	\p\left(\|A_n - Q^TA_nQ\|_F^2 \leq \frac{\alpha_n k n}{3}\right) \leq 2\exp\left(-\alpha_n^2kn/128\right). \label{eq:equation-relax}
	\end{equation}
	Note that for sufficiently large $n$, the condition of the corollary implies that
	$$\frac{6k\log n}{-\log(4p_n(1-p_n))} \leq \frac{6k\log n}{(1/2-p_n)^2} \leq \frac{\alpha_n k n}{3}. $$
	The same condition also implies that $\alpha_n=\omega\left(\sqrt{\log n /n}\right)$, and combining these facts with Equation \eqref{eq:equation-relax}, we have that
	\begin{align*}
	\p\left(\min_{Q\in\Pi_{n,k}}\|A_n - Q^TA_nQ\|_F^2 \leq \frac{-6k\log n}{\log(4p_n(1-p_n))}\ \forall k\in[n]\right) & \leq \sum_{k=2}^n\sum_{Q\in\Pi_{n,k}} \exp\left(-\frac{\alpha_n^2kn}{128}\right)\\
	&\leq  \sum_{k=2}^n\exp\left(k\log n - 4k\log n\right) \\
	&\leq \frac{1}{n^{2}}
	\end{align*}
	for $n$ sufficiently large. As in the proof of Theorem \ref{thm:asym} part i., the result follows from an application of the Borell-Cantelli lemma.
\end{proof}
\subsection{Proof of Corollary \ref{cor:SW}}
\label{sec:pfSW}
\begin{proof}%[{Proof of Corollary \ref{cor:SW}}]
	To prove part a), define $P^{(i)}$ as the permutation that only switches vertices $i$ and $i+1$, i.e., $P^{(i)}_{i,i+1}=P^{(i)}_{i+1,i} = 1$ and $P^{(i)}_{jj}=1$ for $j=\in [n]\setminus\{i,i+1\}$, and let  
	$$
	S_{n}=\{P^{(i)}:i\ \text{mod }2=1\}=\begin{cases}\{P^{(1)},P^{(3)},\cdots,P^{(n-1)}\}\text{ if }n\text{ is even}\\
	\{P^{(1)},P^{(3)},\cdots,P^{(n-2)}\}\text{ if }n\text{ is odd}
	\end{cases}
	$$
	be a set of disjoint permutations of this type. Without loss of generality, take $P=P^{(1)}$. Then
		\begin{align*}
			\frac{1}{4}\|A - P^TAP\|_F^2 & = \sum_{j=3}^n\left(A_{j2} - A_{j1}\right)^2\\
			%& = & (A_{d+2,2} - A_{d+2,1})^2 + (A_{n-d+1,2} - A_{n - d+1,1})^2 + \sum_{j=d+3}^{n-d}\left(A_{j2} - A_{j1}\right)^2\\
			& =  (1 - A_{d_n+2,1})^2 + (A_{n-d_n+1,2} - 1)^2 + \sum_{j=d_n+3}^{n-d_n}\left(A_{j2} - A_{j1}\right)^2,
		\end{align*}
		which is a sum of $n-2d_n$ independent Bernoulli random variables; namely
		\begin{align*}
		  (1-A_{d_n+2,1})^2&\sim\text{Bern}(1-\beta_n);\\
		  (A_{n-d_n+1,2} - 1)^2 &\sim\text{Bern}(1-\beta_n);\\
		  \left(A_{j2} - A_{j1}\right)^2&\sim\text{Bern}(2\beta_n(1-\beta_n))
		  \text{ for }d_n+3\leq j\leq n-d_n.
		\end{align*} If $n-2d_n=O(1)$, then part iii. of Theorem \ref{thm:asym} completes the proof. Hence, consider $n-2d_n=\omega(1)$. Define
		%The first two terms have expected value $2(1-\beta_n)$ while the expectation of the third term is $2(n-2d-2)\beta_n(1-\beta_n)=$. Consider
		\begin{align*}
	f(\beta_n) &:= \frac{1}{4}\Bbb{E}\|AP-PA\|^2_F +  \sqrt{2n\log n}\\
	&=2(1-\beta_n)\left[1+ (n-2d_n-2)\beta_n\right]+  \sqrt{2n\log n}\\
	&=o\left(\frac{\log n}{(1/2-p_n)^2}\right),
		\end{align*}
		%= 2(1-\beta_n)\left[1+ (n-2d-2)\beta_n\right] +  \frac{2\log n}{n-2d_n}.$$
		where the last equality follows from the assumptions on the growth rate of $\beta_n$ and $p_n$. 
		%If the conditions in Corollay \ref{cor:SW} hold, then $f(\beta_n)=o\left(\frac{\log n}{(1/2-p_n)^2}\right)$.
		By Hoeffding's inequality,
		\begin{eqnarray*}
			\p\left(\frac{1}{4}\|A-P^TAP\|^2_F\geq f(\beta_n) \right) \leq \exp\left(-4\log n\right).
		\end{eqnarray*}
		Hence,
		\begin{eqnarray*}
			\p\left(\max_{P\in {S}_{n}}\frac{1}{4}\|A-P^TAP\|^2_F \geq f(\beta_n)\right)  & \leq & \sum_{P\in{S}_n}\p\left(\frac{1}{4}\|A-P^TAP\|^2_F \geq f(\beta_n) \right)\\
			& \leq & \frac{n}{2}\exp\left(-4\log n\right) = \frac{1}{2n^3}.
		\end{eqnarray*}
		By the Borell-Cantelli lemma, for all but finitely many $n$,  $\max_{P\in\mathcal{S}}\frac{1}{4}\|AP-PA\|^2_F < f(\beta_n)=o\left(\frac{\log n}{(1/2-p_n)^2}\right)$, and using Theorem \ref{thm:asym} part ii. (resp., part iii.) if $\min_{P\in\mathcal{S_n}}\frac{1}{4}\|AP-PA\|^2_F(1/2-p_n)^2=\omega(1)$ (resp., $\min_{P\in\mathcal{S_n}}\frac{1}{4}\|AP-PA\|^2_F(1/2-p_n)^2=O(1))$, the result follows.
		
		For part b),  define a $n\times n$ graph $B$ such that $B_{ij}=A_{ij}$ whenever $|i-j| \text{mod}(n-1-d_n)\geq d_n$, and $B_{ij}=B_{ji}\sim\text{Bern}(\beta_{n})$ otherwise. Then $B\sim\text{G}(n,\beta_n)$. 
		Consider a permutation $Q\in\Pi_{n,k}$, and observe that
		\begin{align}
			2\|A-Q^TAQ\|_F^2 & = 2\|B - Q^TBQ + (A-B) - Q^T(A-B)Q\|^2_F\nonumber\\
			& \geq \|B-Q^TBQ\|_F^2 -8kd_n= \|B-Q^TBQ\|_F^2  - o\left(k\frac{\log n}{(1/2-p_n)^2}\right).\label{eq:proof-sw-B}
		\end{align}
		where
		$$2\|(A-B)-Q^T(A-B)Q\|_F^2\leq 8kd_n$$ follows from $(A-B)$ and $Q^T(A-B)Q$ only disagreeing on edges incident to the $k_n$ vertices that are permuted by $Q_n$, and there are at most $2d_n$ such edges per permuted vertex (each being double counted).
		Since $B$ is an ER graph, following the proof of Corollary \ref{cor:examples} mutatis mutatandis, it can be shown that, with probability 1, for all but finitely many $n$ $$\min_{Q\in\Pi_{n,k}}\|B-Q^TBQ\|_F^2\geq  \frac{-12k\log n}{\log(4p_n(1-p_n)},$$
		and thus
		 $$\min_{Q\in\Pi_{n,k}}\|A-Q^TAQ\|_F^2\geq \frac{-6k\log n}{\log(4p_n(1-p_n)}+ o\left(k\frac{\log n}{(1/2-p_n)^2}\right).$$
	The result now follows from Theorem \ref{thm:asym} part i.
\end{proof}

%%%%%%%%%%%%%%%%%%%%%%%%%%%%%%%%%%%%%%%%%%%%%%%%%%%%%%%%%%%%%%%%%%%%%%%%%%%%%%%%%%%%%%%%
\subsection{Proof of Corollary~\ref{cor:correlated}}
\begin{proof}[Proof of Corollary~\ref{cor:correlated}]
 Let $A\sim \text{G}(n,q_n)$ be an Erd\H{o}s-R\'enyi graph. Given a permutation $Q\in\Pi_{n,k}$ that shuffles exactly $k$ vertices, the number of edge disagreements can be expressed as
 \begin{align*}
        \frac{1}{2}\|A-Q^TAQ\|_F^2 %& = \sum_{i=1}^{n-1}\sum_{j=i+1}^n(A_{i,j} - A_{\sigma_Q(i),\sigma_Q(j)})^2\\
        & = \sum_{(i,j)\in \mathcal{U}_Q}(A_{i,j} - A_{\sigma_Q(i),\sigma_Q(j)})^2,
    \end{align*}
     where 
 $$\mathcal{U}_Q := \left\{(i,j)\in\binom{V}{2} : (\sigma_Q(i),\sigma_Q(j)) \neq (i,j)\text{ and }(\sigma_Q(i),\sigma_Q(j)) \neq (j,i)\right\}.$$
 To calculate the number of elements in $\mathcal{U}_Q$, define the sets
    \begin{align*}
        \mathcal{U}_Q^{(1)}& := \left\{(i,j)\in\binom{V}{2} : \sigma_Q(i)\neq i\text{ and }\sigma_Q(j) \neq j\right\},\\
        \mathcal{U}_Q^{(2)}& := \left\{(i,j)\in\binom{V}{2} : \sigma_Q(i)\neq i\text{ and }\sigma_Q(j) = j\right\},\\
        \mathcal{U}_Q^{(3)}& := \left\{(i,j)\in\binom{V}{2} : \sigma_Q(i)=j\text{ and }\sigma_Q(j)=i\right\},\\
    \end{align*}
 and observe that
  $\mathcal{U}_Q = (\mathcal{U}^{(1)}_Q\cup  \mathcal{U}^{(2)}_Q)\setminus  \mathcal{U}^{(3)}_Q$, where $| \mathcal{U}^{(1)}_Q| = \binom{k}{2}$, $| \mathcal{U}^{(2)}_Q|=k(n-k)$, and $| \mathcal{U}^{(3)}_Q|\leq \frac{k}{2}$. Hence, 
  $$|\mathcal{U}_Q| \geq \frac{k(k-2)}{2} + k(n-k) \geq \frac{nk}{3}.$$
    We make use of the following result of \cite{alon} as stated in \cite{lyzinski2014seeded}.
 
    \begin{theorem}[Theorem 3 of \cite{lyzinski2014seeded}] Suppose  $X$ is a function of $\eta$ independent Bernoulli$(q)$ random variables such that changing the value of any one of the Bernoulli random variables changes the value of $X$ by at most 2. For any $t:0\ 0\leq t <\sqrt{\eta q (1-q)}$, we have
    $$\p\left(|X-\mathbb{E}[X]|>4t\sqrt{\eta q(1-q)}\right) \leq 2\exp(-t^2).$$\label{ref:correr}
    \end{theorem}
    Using the notation of the  theorem above, let
    $X=\frac{1}{2}\|A-Q^TAQ\|_F^2$, $\eta = |\mathcal{U}_Q|$, and observe that
    \begin{equation}
    \mathbb{E}[X] = \eta\mathbb{P}(A_{ij}\neq A_{\sigma_Q(i), \sigma_Q(j)}) = 2\eta q_n(1-q_n)\geq \frac{2nkq_n(1-q_n)}{3}. \label{eq:corr-er-EX}    
    \end{equation}
    By the conditions in the corollary, there exists $n_0$ such that for all $n>n_0$
    \begin{equation}
        \frac{nk}{3}q_n(1-q_n) > \frac{6k\log n}{\log\left(\frac{1}{4\tilde{p}_n(1-\tilde{p}_n)}\right)}.\label{eq:corr-er-nk}
    \end{equation}
    Observe that the assumptions of Theorem~\ref{ref:correr} hold for any $0\leq t<\sqrt{\eta q_n(1-q_n)}$, and in particular let $t=\frac{1}{4}\sqrt{\frac{nkq_n(1-q_n)}{3}}$. Then, using Theorem~\ref{ref:correr} in combination with  Equations~\eqref{eq:corr-er-EX} and \eqref{eq:corr-er-nk},
    \begin{align*}
        \p\left(X \leq \frac{6k\log n}{\log\left(\frac{1}{4\tilde{p}_n(1-\tilde{p}_n)}\right)}\right) & \leq \p\left(|X - \mathbb{E}[X]| \geq \mathbb{E}[X] - \frac{6k\log n}{\log\left(\frac{1}{4\tilde{p}_n(1-\tilde{p}_n)}\right)}\right)\\
        & \leq \p\left(|X - \mathbb{E}[X]| \geq \frac{1}{3}nkq_n(1-q_n)\right)\\
        & \leq 2\exp\left(-\frac{nkq_n(1-q_n)}{48}\right)
    \end{align*}
    By the assumptions in Corollary~\ref{cor:correlated}, there exists $n_1$ sufficiently large such that for all $n\geq n_1$, it holds that $q_n(1-q_n) > \frac{192\log n}{n}$. Therefore, for all $n>\max\{n_0,n_1\}$,
\begin{align*}
    \p\left(\min_{Q_n\in\Pi_{n,k}}\|A-Q^TAQ\|_F^2  \leq  \frac{-6k\log n}{\log(4\tilde{p}_n(1-\tilde{p}_n))}\ \forall k\geq 2 \right) & \leq  \sum_{k=2}^n\sum_{Q\in\Pi_{n,k}}2\exp\left(\frac{-nkq_n(1-q_n)}{48}\right)\\
     & \leq  \sum_{k=2}^n2\exp\left(k\log n -4k\log n\right)\\
     & \leq \frac{2}{n^2}.
\end{align*}
As in the proof of Theorem \ref{thm:asym-het} part i., the result follows from an application of the Borell-Cantelli lemma.

\end{proof}

%%%%%%%%%%%%%%%%%%%%%%%%%%%%%%%%%%%%%%%%%%%%%%%%%%%%%%%%%%%%%%%%%%%%%%%%%%%%%%%
\subsection{Proof of Theorem~\ref{thm:asym2}}
\label{sec:pfthmasymp2}
\begin{proof}
The proof of this result follows  the proof of Theorem \ref{thm:asym} mutatis mutandis by only considering permutations of enough vertices. In particular, the conditions of part i. imply that an analogous to Equations~\eqref{eq:proof-asymp-1-to-2a} and \eqref{eq:proof-asymp-1-to-2b} hold, that is,
\begin{align*}
\p\left( \argmin_{P\in\Pi_n}\|A_n-P^TB_nP\|_F^2\notin \bigcup_{k=k_n}^n \Pi_{n,k} \right) & = \p(\exists \,Q\in \bigcup_{k=k_n}^n \Pi_{n,k}\text{ s.t. }\widehat{Y}_Q\leq 0)\\
& \leq \sum_{k=k_n}^n \exp(-2k\log n)\\
\\
&\leq n^{-3}.
\end{align*}
Therefore, with probability $1$ it holds that for all but finitely many $n$ (by the Borel-Cantelli lemma as $n^{-3}$ is finitely summable) the  MLE of $P_n$ is in $\Pi_n \setminus \bigcup_{k=k_n}^n\Pi_{n,k}$, so $\max\|\hpi - P_n\|_F^2\leq k_n$, and because $\alpha_n = \omega(k_n)$, the strong $\alpha_n$-consistency follows.

\end{proof}
%%%%%%%%%%%%%%%%%%%%%%%%%%%%%%%%%%%%%%%%%%%%%%%%%%%%%%%%%%%%%%%%%%%%%%%%%%%%%%%
\subsection{Proof of Corollary \ref{cor:example3}}
\label{sec:pfexample3}
\begin{proof}%[Proof of Corollary \ref{cor:example3}]
	Let $H\sim$G$(n,d/n)$ be an ER graph with rate $d/n$.
	Define the events
	\begin{align*}
	\mathcal{A}_{n,d}&=\{  \text{ Eq. }\ref{eq:growthrate-quasi}\text{ holds for }A\text{ for all }k=\omega(n^{2/3})\}\\
	\mathcal{A}_{n,d/n}&=\{  \text{ Eq. }\ref{eq:growthrate-quasi}\text{ holds for }H\text{ for all }k=\omega(n^{2/3})\}\\
	\mathcal{R}_{n,d/n}&=\{H\text{ is }d\text{-regular} \}
	\end{align*}
	We then have that
	\begin{align*}
	\p(\mathcal{A}_{n,d/n}^c)&\geq \p(\mathcal{A}_{n,d/n}^c|\mathcal{R}_{n,d/n})\p(\mathcal{R}_{n,d/n}) =\p(\mathcal{A}_{n,d}^c)\p(\mathcal{R}_{n,d/n}).
	\end{align*}
	Note that for $H_n$ in $\Pi_{n,k_n}$ with $k_n=k=\omega(n^{2/3})$, 
	\begin{itemize}
	    \item[i.] $\|H_n - Q^TH_nQ\|^2_F$ is a function of $\binom{k}{2}+(n-k)k-O(k)$ independent Bern($p=d/n$) random variables;
	    \item[ii.] Changing any one of these Bern($p=d/n$) random variables can change $\|H_n - Q^TH_nQ\|^2_F$ by at most 2;
	    \item[iii.] We have that 
	    \begin{align*}
\mu_Q&=\e(\|H_n - Q^TH_nQ\|^2_F)=2\left(\binom{k}{2}+(n-k)k-O(k)\right)\frac{d}{n}\left(1-\frac{d}{n}\right)\\
&=(1+o(1))nk\frac{d}{n}\left(1-\frac{d}{n}\right).
	    \end{align*}
	\end{itemize}
	We then have (where $c$ is a constant that can change line-to-line)
	\begin{align*}
	\p(\mathcal{A}_{n,d/n}^c) & \leq  \p \left(\|H_n - Q^TH_nQ\|^2_F < \frac{6k\log n}{-\log(4p_n(1-p_n))}\right)\\
	&=  \p \left(\|H_n - Q^TH_nQ\|^2_F-\mu_Q < \frac{6k\log n}{-\log(4p_n(1-p_n))}-(1+o(1))kd\left(1-\frac{d}{n}\right)\right)\\
	& \leq  \p \left(\|H_n - Q^TH_nQ\|^2_F-\mu_Q < \frac{6k\log n}{c(1/2-p_n)^2}-(1+o(1))kd\left(1-\frac{d}{n}\right)\right)
	\end{align*}
	An appropriate choice of $C$ in the conditions of the corollary combined with an application of McDiarmind's inequality then yields (where $c$ is a constant that can change line-to-line)
		\begin{align*}
	\p(\mathcal{A}_{n,d/n}^c) & \leq  \p \left(\|H_n - Q^TH_nQ\|^2_F-\mu_Q < -10kd\left(1-\frac{d}{n}\right)\right)\\
	&\leq \text{exp}\left\{
	\frac{-ck^2d^2\left(1-\frac{d}{n}\right)^2  }
	{ (1+o(1))nk\frac{d}{n}\left(1-\frac{d}{n}\right)}
	\right\}\\
	&\leq \text{exp}\left\{-ckd
	\right\}
	\end{align*}
	Hence, by Lemma 4.1.iv in \cite{kim}, for $n$ sufficiently large
	$$\p(\mathcal{R}_{n,d/n})\geq \text{exp}(-nd^{1/2+\epsilon/2}), $$
	and so for $n$ sufficiently large,
	\begin{align*}
	\p(\mathcal{A}_{n,d}^c)&\leq \p(\mathcal{R}_{n,d/n})^{-1}\p(\mathcal{A}_{n,d/n}^c) \leq \text{exp}\left\{nd_n^{1/2+\epsilon/2}-\omega(n^{2/3}d_n )\right\}\\
	&\leq \text{exp}\Big\{\underbrace{\big(n-\omega(n^{2/3}d_n^{1/2-\epsilon/2} )}_
	{=n-\omega(n)=-\omega(n)  }
	\big)d_n^{1/2+\epsilon/2}\Big\}.
	\end{align*}
	As this probability is summable, we have that, with probability 1, $\mathcal{A}_{n,d}$ occurs for all but finitely many $n$ by the Borel-Cantelli lemma.
	Therefore, with probability 1 the conditions of Theorem \ref{thm:asym2} are met for any sequence $\alpha_n = \omega(n^{2/3})$ for all but finitely many $n$ and the strong consistency of the MLE follows.
\end{proof}

\bibliographystyle{apa}
\bibliography{biblio}

\end{document}